\documentclass{article}

\usepackage{enumitem}
\setlist[itemize]{noitemsep, topsep=-5pt}
\setlist[enumerate]{noitemsep, topsep=-5pt}

\PassOptionsToPackage{numbers, compress}{natbib}
\usepackage[final]{neurips_2025}

\usepackage[utf8]{inputenc} % allow utf-8 input
\usepackage[T1]{fontenc}    % use 8-bit T1 fonts
\usepackage{hyperref}       % hyperlinks
\usepackage{url}            % simple URL typesetting
\usepackage{booktabs}       % professional-quality tables
\usepackage{amsfonts}       % blackboard math symbols
\usepackage{nicefrac}       % compact symbols for 1/2, etc.
\usepackage{microtype}      % microtypography
\usepackage{xcolor}         % colors

\usepackage{mathtools}
\usepackage{algorithm}

\usepackage{algorithmic}
\usepackage{tikz}
\usepackage{caption}
\usepackage{subcaption} % Pour gérer les sous-figures
\usepackage{amsfonts, amsmath, amsthm}       % blackboard math symbols
\usepackage{amssymb}  % besoin poru nleq
\usepackage{multirow} % Pour fusionner les cellules sur plusieurs lignes
\usepackage{booktabs} % Pour des lignes horizontales de meilleure qualité
\usepackage{accents}
\usepackage{comment}
\usepackage{wrapfig}
\usepackage{enumitem}
\usepackage{array}
\usepackage{tabularx}
\usepackage{wrapfig}

\usepackage{natbib} 
\newcommand{\AR}[1]{{\color{magenta}AR: #1}}

\newcommand{\AJ}[1]{{\color{blue}AJ: #1}}

\newtheorem{lemma}{Lemma}
\newtheorem{definition}{Definition}
\newtheorem{theorem}{Theorem}

\newtheorem{remark}{Remark}

\newtheorem*{theorem*}{Theorem}
\newtheorem*{proposition*}{Proposition}

\newcommand{\Lnorm}{{L}}
\newcommand{\Lcombi}{\mathcal{L}}

\newcommand{\Qnorm}{{Q}} %Q_{norm} peut etre envisagé
\newcommand{\Qcombi}{{\mathcal{Q}}}

\newcommand{\Prao}{{P_{opt}}} %{P_{rao}}

\newcommand{\Ploukas}{{P_{Loukas}}} %{P_{L}}
\newcommand{\diagmatrix}{{\Delta}}

\newcommand{\preservedspace}{\mathcal{R}}

\title{Taxonomy of reduction matrices for Graph Coarsening}

\author{%
  Antonin Joly \\
  CNRS, IRISA, Rennes, FRANCE\\
  \texttt{antonin.joly@irisa.fr} \\
  \And
  Nicolas Keriven\\
  CNRS, IRISA, Rennes, FRANCE\\
  \texttt{nicolas.keriven@cnrs.fr} \\
  \And
  Aline Roumy \\
  INRIA, Rennes, FRANCE\\
  \texttt{aline.roumy@inria.fr}
}

\begin{document}

\maketitle
\begin{comment}
\AJ{Proposition de titre :
\begin{itemize}
    \item Taxonomy of reduction matrices for Graph Coarsening % +1 NK
    \item Taxonomy of reduction matrices under Graph Coarsening structural constraints
    \item 
\end{itemize}
}
\AR{Proposition de titre:\\
Revisiting Graph Coarsening: from lifting-based definition to reduction optimization
}
\end{comment}

\begin{abstract}
Graph coarsening aims to diminish the size of a graph to lighten its memory footprint, and has numerous applications in graph signal processing and machine learning. It is usually defined using a reduction matrix and a lifting matrix, which, respectively, allows to project a graph signal from the original graph to the coarsened one and back. This results in a loss of information measured by the so-called Restricted Spectral Approximation (RSA). Most coarsening frameworks impose a fixed relationship between the reduction and lifting matrices, generally as pseudo-inverses of each other, and seek to define a coarsening that minimizes the RSA.

In this paper, we remark that the roles of these two matrices are not entirely symmetric: indeed, putting constraints on the \emph{lifting matrix alone} ensures the existence of important objects such as the coarsened graph's adjacency matrix or Laplacian.
In light of this, in this paper, we introduce a more general notion of reduction matrix, that is \emph{not} necessarily the pseudo-inverse of the lifting matrix. 
We establish a taxonomy of ``admissible'' families of reduction matrices, discuss the different properties that they must satisfy and whether they admit a closed-form description or not. We show that, for a \emph{fixed} coarsening represented by a fixed lifting matrix, the RSA can be \emph{further} reduced simply by modifying the reduction matrix. We explore different examples, including some based on a constrained optimization process of the RSA. Since this criterion has also been linked to the performance of Graph Neural Networks, we also illustrate the impact of this choices on different node classification tasks on coarsened graphs.

%degree of freedom in the choice of the reduction matrix by establishing a taxonomy of admissible reduction matrices for a fixed lifting matrix and we offer a more flexible framework for its choice. Accordingly, we discuss the performances of different choices of reduction matrix for a given spectral preservation guarantee: the Restricted spectral approximation(RSA). \AJ{This spectral preservation guarantee being linked to Graph Neural Networks (GNN), we conduct node classification task on real data and observe that this spectral guarantee is not the only factor impacting their result.}
\end{abstract}

\section{Introduction}

In recent years, several applications in data science and machine learning have produced large-scale \emph{graph} data \cite{hu2020open, Bronstein2021}. For instance, online social networks \cite{Ediger2010} or recommender systems \cite{Wang2018} routinely produce graphs with millions of nodes or more. To handle such massive graphs, researchers have developed general-purpose \emph{graph reduction} methods \cite{bravo2019unifying}, such as \textbf{graph coarsening} \cite{loukas2019graph, chen2022graph} as well as specific learning techniques on these coarsened graphs \cite{joly2024graph,huang2021scaling}. Graph coarsening starts to play an increasingly prominent role in machine learning applications \cite{chen2022graph}.

\paragraph{Graph Coarsening and Spectral guarantees}
Graph coarsening consists in producing a small graph from a large graph %. It is generally described by a reduction and a lifting matrix 
while retaining some of its key properties.  %proceed similarly for the opposite way, from the coarsened graph to the original. 
There are many ways to evaluate the quality of a coarsening, following different criteria \cite{dhillon2007weighted, loukas2019graph, chen2022graph}. The majority of these approaches aims to preserve spectral properties of the graph and its Laplacian, and have given rise to different coarsening algorithms \cite{loukas2019graph, chen2023gromov, bravo2019unifying, jin2020graph, loukas2018spectrally}. The most widely used spectral guarantee is the so-called \emph{Restricted Spectral Approximation} (RSA, see Sec.~\ref{sec:rappel}), introduced by Loukas \cite{loukas2019graph}. In broad terms, the RSA states that the frequency content of a certain subspace of graph signals is approximately preserved when projected on the coarsened graph and then re-lifted in the original one, or intuitively, that the coarsening is well-aligned with the low-frequencies of the Laplacian. The RSA is a general-purpose criterion with many applications, from clustering to signal reconstruction \cite{loukas2019graph, joly2024graph}.
Recently, RSA guarantees were also used to guarantee the performances of Graph Neural Networks on the coarsened graph \cite{joly2024graph}. %has translated the  to message passing guarantees for Graph Neural Network (GNNs) on the coarsened graph. Indeed, Graph Neural networks \cite{scarselli2008graph, kipf2016GCN, Bronstein2021} relies on a message passing paradigm \cite{Gilmer2017} that can be represented by the multiplication of the node representation by a given propagation matrix. This paper proposed a new propagation matrix, unusually asymmetric, that depends on both the reduction and the lifting matrix. 

\paragraph{Projection and lifting.} The projection and re-lifting operations are generally described by two matrices: the \emph{reduction matrix} allows to transform graph signals from the original graph to the coarsened one, while the \emph{lifting matrix} does the opposite. In virtually all works on graph coarsening, these two matrices are simply pseudo-inverse of each other, and both can represent the ``graph coarsening'' indifferently. However, in this paper we make the following remark: \emph{their roles are not entirely symmetric}. Mathematically, as we will see, the \emph{lifting matrix alone} has to be quite constrained for the graph coarsening to be ``well-defined'', with a consistent adjacency matrix and Laplacian. The reduction matrix, on the other hand, does not seem to play a role in the \emph{definition} of graph coarsening. However, it \emph{does} play a role in computing the RSA. Therefore, in this paper, we examine the following questions: for a \emph{fixed} lifting matrix, 1) \textbf{What are the admissible degrees of freedom for the reduction matrix?} and 2) \textbf{Can we improve the RSA by simple modification of the reduction matrix alone?} 

\paragraph{Contribution.} In this paper, we thus define and then explore the admissible sets of reduction matrices over which to optimize the RSA. We introduce several interesting examples, from closed-form ones motivated by notions of optimality and memory footprint to optimization-based ones over well-defined sets with various properties. We compare these different choices, both in terms of RSA and performance when used within GNNs trained on coarsened graphs.
%of reduction matrix to minimize the RSA from classical choices to optimization based matrices, including a choice moti 

%\paragraph{Contribution} %While most of the algorithms approach focus on defining the reduction matrix and impose a fixed lifting matrix which depends on it (classically the Moore Penrose inverse), we show that \textbf{only the lifting matrix guarantees the graph structure of the coarsened graph}. We might naturally wonder :  \textbf{Could we use the liberty degree over the choice of the reduction matrix to minimize the RSA for a fixed lifting matrix ?}. To this end, 
%We define and then explore the admissible sets of research for the reduction matrix to minimize the RSA. To our knowledge, we are the first to
%treat the lifting matrix and the reduction matrix as two different operation and
%take advantage of the liberty degree over the reduction matrix to minimize the RSA and not fix it as the Moore Penrose pseudo inverse or the transposed of the lifting matrix. 
%We then empirically compare different choices of reduction matrix to minimize the RSA from classical choices to optimization based matrices, including a choice moti 

\paragraph{Related work} 
Graph coarsening is derived from the multigrid-literature \cite{ruge1987algebraic} and belongs to a broader class of methods commonly referred to as \emph{graph reduction}. The latter includes graph sampling \cite{Hu2013}, %which consists in sampling nodes to extract a subgraph; 
graph sparsification \cite{spielman2011spectral, allen2015spectral, lee2018constructing}, %that creates sparser graph by removing edges; 
or more recently graph distillation \cite{jin2021graph, zheng2024structure, jin2022condensing}, %which extends some of these principles by authorizing additional informations 
inspired by dataset distillation \cite{wang2018dataset}.
%Graph Coarsening originates from the multigrid-literature \cite{ruge1987algebraic}, and is part of a family of methods commonly referred to as \emph{graph reduction}, which includes graph sampling \cite{Hu2013}, which consists in sampling nodes to extract a subgraph; graph sparsification  \cite{spielman2011spectral, allen2015spectral, lee2018constructing}, that focuses on eliminating edges; or more recently graph distillation \cite{jin2021graph, zheng2024structure, jin2022condensing}, which extends some of these principles by authorizing additional informations inspired by dataset distillation \cite{wang2018dataset}.
%
Some of the first coarsening algorithms were linked to the graph clustering community, e.g. \cite{defferrard2016convolutional} which used recursively the Graclus algorithm \cite{dhillon2007weighted} algorithm. % itself built on Metis \cite{karypis1998fast}. 
Linear algebra technics such as the Kron reduction were also employed \cite{loukas2019graph} \cite{dorfler2012kron}. In \cite{loukas2019graph}, the author introduces the RSA, and presents a greedy algorithm that recursively merge nodes by optimizing some cost, which in turns leads to RSA guarantees.
%with the purpose of preserving a spectral property of the coarsened Laplacian, namely the Restricted Spectral Approximation. 
This is the approach we use in our experiments (Sec.~\ref{sec:experiments}). It was followed by several similar methods with the same spectral criterion \cite{chen2023gromov, bravo2019unifying, jin2020graph, loukas2018spectrally}. Since modern graph often includes node features, other approaches proposed to take them into account in the coarsening process, often by learning the coarsening with specific regularized loss \cite{kumar2023featured, ma2021unsupervised,dickens2024graph}. %While these works often seek to preserve a fixed number of node features as in e.g. \cite{kumar2023featured,dickens2024graph})
On the contrary, the RSA guarantees \cite{loukas2019graph} leveraged in this paper are \emph{uniform} over a whole subspace to ensure the spectral preservation of the coarsened graph. 

Closer to us, some works aim to optimize various quantities after the coarsening has been computed. For instance, GOREN \cite{cai2021graph} optimizes in a data-driven manner the edges' weights in the coarsened graph, which is quite different from focusing on reduction/lifting matrices as proposed here. Moreover, we consider the RSA, a general-purpose criterion not necessarily related to downstream tasks. The literature also includes different choices of reduction/lifting matrices, or propagation matrices in GNNs on coarsened graphs \cite{joly2024graph}, but to our knowledge this paper is the first to put forth the idea of \emph{decorrelating reduction and lifting matrices} up to a certain point, with precise mathematical definitions of the consequences.

%GNNs can also be used to produce a data-driven coarsening with GOREN \cite{cai2021graph}. This method aims for downstream tasks by learning new edge weights for the coarsen graph. As it modifies the coarsened adjacency matrix, it also changes the lifting matrix and thus it is quite different from our approach which considers a fixed lifting matrix.  \AJ{A modifier : pour l'instant copié collé neurips sauf GOREn 2024} 

Finally, we mention \emph{graph pooling}, which is designed to mimick the pooling process in deep convolutional models on images and is somewhat related to graph coarsening in terms of vocabulary. One difference is that graph pooling tends to focus only on the reduction phase while graph coarsening focuses on repeated reduce-then-lift operations between the coarsened and original graphs. Although some pooling method can be computed as preprocessing such GRACLUS \cite{dhillon2007weighted}, %introduced by \cite{defferrard2016convolutional} and then be considered as graph coarsening, 
the most well-known pooling methods are data-driven and fully differentiable (Diffpool \cite{ying2018hierarchical}, top-K pooling \cite{gao2019graph}, DMoN \cite{tsitsulin2023graph}).  These methods use interchangeably the reduction and the lifting matrix by choosing one as the \emph{transposed} of the other. 
%Similarly to graph coarsening, they are used to compute the features of the coarsened graph and its adjacency matrix. 
Usually in graph pooling, this matrix is unconstrained and is either defined heuristically or learned, while graph coarsening proposes mathematical links between the reduction and the lifting matrix. In this paper, we explore these mathematical links.
%The distinction between graph coarsening and graph pooling is subtle, we propose a clearer positioning with respect to graph pooling and its recent redefinition in \cite{bianchi2024intropool, grattarola2022understanding}. This is detailed in the appendix App.~\ref{app:pooling}.

%\NK{pour l'instant le related work est un peu trop ``général'': essaye de te concentrer sur les aspects reliés au papier, càd comment les gens calculent/optimizent différentes quantités liées au coarsening (typiquement les poids dans goren), et à quel point les matrices Q et P étaient toujours pseudo-inverse l'une de l'autre jusqu'à présent, ou non, et si non quelles sont les différences avec toi. (je donne un example avec le paragraph goren, mais si je dis pas de bêtise le pooling implique aussi des "choix" de matrices)}

\paragraph{Outline.} We start by background material on graph coarsening in Sec.~\ref{sec:rappel}, highlighting the roles of the lifting and reduction matrices. %for the coarsened graph and the constraints associated to this matrix 
We emphasize the asymmetricity of their roles, along with the strong constraints put on the lifting matrix alone. %This shed light on the \emph{a-priori} absence of constraints on the reduction matrix, which gives us a liberty degree to minimize the RSA that we define also in that section. 
Then in Sec.~\ref{sec:pset}, we study sets of admissible reduction matrices, focusing on various notions of ``generalized inverses''. Some will be very generic, while others will admit parametrizations that are convenient for optimization. %for a defined lifting matrix and some particular subsets better adapted for an optimization process. 
In Sec.~\ref{sec:pexamples}, we then study several motivated examples of reduction matrices, classical or entirely novel, with an analytical closed-form expression or based on optimization procedures. We relate them to the properties defined in the section before.
%
%propose several optimization procedure and some well chosen but also naive choice of reduction matrix and place them in the corresponding subset before 
Finally, we compare their performance in terms of RSA or GNN performance in Sec.~\ref{sec:experiments}. %Motivated by a work that links the RSA and GNNs on the coarsened graph we also perform node classification tasks on real data in Sec.~\ref{sec:experiments} .
The code is available at \url{https://gitlab.inria.fr/anjoly/taxonomy-coarsening-matrices},
and proofs are deferred to App.~\ref{app:proof}.

%\section{Differences Between Graph Coarsening and Dimensionality Reduction}
\section{Characterizing graph coarsening with the lifting matrix}
\label{sec:rappel}
\paragraph{Notations} A graph $G$ with $N$ nodes is described by its weighted adjacency matrix $A \in \mathbb{R}^{N\times N}$. The combinatorial Laplacian is defined as $ \Lcombi = \Lcombi(A) := D - A$, where $D = D(A) := \text{diag}(A1_N)$ is the diagonal matrix of the degrees. A matrix is said to be \textbf{binary} if all its coefficients are either $0$ or $1$. For a symmetric positive semi-definite (p.s.d) matrix $M$, we denote $\|x\|_M = \sqrt{x^\top M x}$ the Mahalanobis semi-norm associated with $M$. 

%\paragraph{Notations} A graph $G$ with $N$ nodes is described by its weighted adjacency matrix $A \in \mathbb{R}^{N\times N}$. The combinatorial Laplacian is defined as $ \Lcombi = D - A$. $D = D(A) := \text{diag}(A1_N)$ being the diagonal matrix of the degrees. A matrix is said to be \textbf{binary-valued} if all its coefficients are equal either $0$ or $\alpha \in \mathbb{R}$. For a symmetric positive semi-definite (p.s.d) matrix $M$, we denote $\|x\|_M = \sqrt{x^\top M x}$ the Mahalanobis semi-norm associated to $M$. 

\paragraph{Coarsening} The goal of coarsening is to reduce the size of a graph $G$ with $N$ nodes to a coarsened graph $G_c$ with $n < N$ nodes. The proportion of reduction achieved is measured by the \textbf{coarsening ratio} $r = 1- \frac{n}{N}$. The mapping from $G$ to $G_c$ is obtained by grouping set of nodes in $G$ to form supernodes in $G_c$. This mapping can be represented by a matrix $\Qcombi\in {\mathbb{R}}^{N\times n}$, where $\Qcombi_{ik}>0$ means that node $i$ from $G$ has been mapped onto the supernode $k$ of $G_c$. %Generally the groups are disjoint, 
To represent a true mapping from the original nodes to the supernodes, the matrix $\Qcombi$ needs to be \textbf{well-partitioned}. %, which is formally defined below.
\begin{definition}[Well-partitioned $\Qcombi$ matrix]
     $\Qcombi$ is said to be well-partitioned if it has exactly \textbf{one} non-zero coefficient per row. %$\Qcombi_{ik} \neq 0 \Longrightarrow \forall k' \neq k, \Qcombi_{ik'} = 0$.
\end{definition}
One natural way \cite{loukas2019graph} to define the adjacency matrix of the coarsened graph is then to take
\begin{equation}\label{eq:ac}
    A_c = \Qcombi^\top A \Qcombi
\end{equation}
%Indeed, with this definition $A_c$ links the supernodes by weighted edges whose values are the number of edges between the grouped original nodes. 
In addition to be well-partitioned, it is also natural to impose that $\Qcombi$ is binary, as shown by \cite[Prop.7]{loukas2019graph} in the following Lemma. 
\begin{lemma}[{\cite{loukas2019graph}}]\label{lem:loukas}
Let $\Qcombi$ be a well-partitioned matrix. The two following properties are equivalent: (i) $\Qcombi$ is proportional to a binary matrix; 
(ii) for all $A$, we have $\Lcombi_c := \Lcombi(A_c) = \Qcombi^\top \Lcombi(A) \Qcombi$.
\end{lemma}
In other words, the Laplacian of the coarsened graph can be defined by the same equation as \eqref{eq:ac}. In light of this, matrices $\Qcombi$ in this paper will always be well-partitioned and binary. This results in a particularly interpretable $A_c$: a weighted edge between two supernodes has a value equal to the sum of the weights of all the edges between the two groups of original nodes.

\paragraph{Lifting, Reduction, and spectral quality measure for Graph Coarsening} The quality of a coarsening can be assessed from a signal processing point of view \cite{loukas2019graph}. Indeed, one way to interpret $\Qcombi$ is that it can be used to ``lift'' a signal $y \in \mathbb{R}^n$ from the coarsened graph to the original one, as $x = \Qcombi y$, and is thus called the \textbf{lifting matrix} in the literature. Its counterpart is a \textbf{reduction matrix} $P \in \mathbb{R}^{n \times N}$ that reduces a signal from $G$ to $G_c$. More formally, let $x \in \mathbb{R}^N$ be a signal over the nodes of $G$. The coarsened signal $x_c\in \mathbb{R}^{n}$ and the re-lifted signal $\tilde x \in \mathbb{R}^{N}$ are defined by
%Before exploring the role of the reduction matrix $P$ we will introduce the classic quality measure of a coarsening. Most of proposed quality measure quantifies the modification of the spectral properties of the graph represented by its combinatorial Laplacian $\Lcombi$. The most adopted quality measure, introduces by Loukas \cite{loukas2019graph}, establish a near isometry property for graph signals with respect to the norm $\|\cdot\|_\Lcombi$, which can be interpreted as a measure of the smoothness of a signal across the graph edges. Given a signal $x \in \mathbb{R}^N$ over the nodes of $G$, we define the coarsened signal $x_c\in \mathbb{R}^{n}$ and the re-lifted signal $\tilde x \in \mathbb{R}^{N}$ by
\begin{equation}\label{eq:xc}
    x_c = P x, \qquad \tilde x = \Qcombi x_c = \Pi x
\end{equation}
%by how it alters the spectral properties of the graph.
%
where $\Pi = \Qcombi P$. To measure the quality of the coarsening, a popular criterion introduced by Loukas \cite{loukas2019graph}, is then the \emph{Restricted Spectral Approximation} (RSA), which measures the loss of information from $x$ to $\tilde x$. 
%a distance between the coarsening and an isometry for signals defined on the graph. The distance is with respect to the norm $\|\cdot\|_\Lcombi$, which can be interpreted as a measure of the smoothness of a signal across the graph edges. 
Since $\Pi$ is at most of rank $n <N$, only a subspace $\preservedspace$ of $\mathbb{R}^N$ may be preserved. This leads to the definition of the RSA constant below. %, which we aim to minimize when optimizing the coarsening.
%Loukas \cite{loukas2019graph} then introduces the notion of \emph{Restricted Spectral Approximation} (RSA) of a coarsening algorithm, which measures how much the projection $\Pi$ is close to the identity for a class of signals. Since $\Pi$ is at most of rank $n <N$, this cannot be true for all signals, but only for a restricted subspace $\preservedspace \subset \mathbb{R}^N$. With this in mind, the \emph{RSA constant} is defined as follows.
\def\RSAconst{\epsilon_{L,Q,P,\preservedspace}}
\begin{definition}[Restricted Spectral Approximation constant]
    Consider a subspace $ \preservedspace \subset \mathbb{R}^N$, a Laplacian $\Lcombi$, a lifting matrix $\Qcombi$, a reduction matrix $P$. % and its corresponding projection operator $\Pi = \Qcombi P$. 
    The \emph{RSA constant} $\epsilon_{\Lcombi,\Qcombi,P,\preservedspace}$ is defined as
    \begin{equation}\label{eq:rsa}
        \epsilon_{\Lcombi,\Qcombi,P,\preservedspace} = \sup_{x\in \preservedspace, \|x\|_\Lcombi =1} \lVert x - \Qcombi P x  \rVert_{\Lcombi}
    \end{equation}
\end{definition}
Classically, the preserved subspace $\preservedspace$ is spanned by the eigenvectors of the first eigenvalues of $\Lcombi$. In this case, the RSA constant can be used to bound the deviation between the spectrums of $\Lcombi$ and $\Lcombi_c$. %Note that, as of now we have not discussed any specific properties of $P$. We do so below. %We now aim to generalize this spectral RSA measure to a broader class of Laplacian matrices.
When $\preservedspace$ is an eigen-subspace of $\Lcombi$, the RSA has an explicit expression \cite{loukas2019graph}: $ \epsilon_{\Lcombi,\Qcombi,P,\preservedspace} = \lVert \Lcombi^{1/2}(I_N - P\Qcombi)VV^T\Lcombi^{+1/2}\rVert_2 $ where $V$ is an orthogonal basis of $\preservedspace$ and $\lVert \cdot \rVert_2$ is the spectral norm. Note that this expression is convex in $P$, which will be useful for optimization. This definition slightly differs from \cite{loukas2019graph}, as it disentangles the roles of the lifting matrix $\Qcombi$ and the reduction matrix $P$, while in \cite{loukas2019graph}, $\Qcombi$ was fixed as the Moore Penrose inverse of $P$.
%\AJ{This definition slightly differs from \cite{loukas2019graph}, \AR{as it disentangles} by disentangling the roles of the lifting matrix $\Qcombi$ and the reduction matrix $P$, while in \cite{loukas2019graph}, $\Qcombi$ was fixed as the Moore Penrose inverse of $P$.}

\paragraph{$\Qcombi$ is more constrained than $P$} One might have noticed that $P$ and $\Qcombi$ do not exactly play symmetric roles. Indeed, and this is the first key message of the paper: as shown by \eqref{eq:ac} and Lem.~\ref{lem:loukas}, \textbf{the matrix $\Qcombi$ alone fully characterizes the graph coarsening}, and \emph{the matrix $\Qcombi$ alone must respect strong constraints} (it must be well-partitioned and binary). Technically, the reduction matrix $P$ does not play any role in the definition of $A_c$ or $\Lcombi_c$.
%Classically \cite{loukas2019graph}, coarsening is described via a \textbf{reduction} matrix $P \in {\mathbb{R}}^{n\times N}$, (that maps each node of $G$ onto a supernode of $G_c$), and a lifting matrix $\Qcombi$, with a predefined link to $P$. 

However, $P$ still plays a role in the computation of the RSA constant. As mentioned in the introduction, in virtually every formulations of graph coarsening, $P$ is taken as the Moore-Penrose pseudo-inverse $P = \Qcombi^+$. For a well-partitioned, binary $\Qcombi$, this matrix has the same support as $Q^\top$, with rows that contain only coefficients $1/n_k$ where $n_k$ is the size of the $k$th supernode. In this paper, we challenge this choice and argue that \emph{there is no real reason for it}.
As we will see, there is a relative degree of freedom in designing $P$, and this is our second main message: \emph{for a fixed $\Qcombi$, the matrix $P$ can be optimized to improve the RSA constant}. %quality of the coarsening}.
Of course, this optimization may still satisfy important constraints in terms of interpretability, feasibility, memory footprint, or just simplicity, and this forms the main questions mentioned in the introduction: given a well-partitioned and binary lifting matrix, what are the ``valid'' reduction matrices? Is there a more ``optimal'' choice to minimize the RSA?

%In other words, the RSA constant measures how much signals in $\preservedspace$ are preserved by the coarsening-lifting operation, with respect to the norm $\|\cdot\|_\Lcombi$. In practice, $\preservedspace$ is often chosen to be the sub-space spanned by the eigenvectors associated to the smallest eigenvalues of $\Lcombi $.
%
%Given some $\preservedspace$ and Laplacian $\Lcombi$, the goal of a coarsening algorithm is then to produce a lifting matrix $\Qcombi $ and reduction matrix $P$ \emph{with the smallest RSA constant possible}. As we have seen $\Qcombi$ has to satisfy constraints to preserve the graph structure but it is interesting to wonder if the liberty degree on $P$ could be used to minimize the RSA value.

\paragraph{Normalized Laplacian matrices} 
Before moving on to the next section, we adapt the previous discussion to a broader notion of \emph{``normalized'' Laplacian},
%We now consider a general Laplacian
that we call $\diagmatrix$-Laplacian, defined as 
\begin{equation}\label{eq:delta_l}
    \Lnorm = \Lnorm(A) = \diagmatrix \Lcombi \diagmatrix  
\end{equation}
where $\Lcombi = \Lcombi(A)$ is the combinatorial Laplacian, and $\diagmatrix = \diagmatrix(A) \in\mathbb{R}^{N\times N} $ a strictly positive diagonal matrix that depends on the adjacency matrix. % $ \diagmatrix= \diagmatrix(A)$. 
This $\diagmatrix$-Laplacian thus encompasses the combinatorial Laplacian when $\diagmatrix = I_N$, or the classical normalized Laplacian when $\diagmatrix = {D}^{-1/2}$. Another interesting example is the self-loop normalized Laplacian $\diagmatrix  = {(D + I_N)}^{-1/2}$, which is such that $\Lnorm = I_N -S$ where $S$ is the propagation matrix of the classical Graph convolution network defined by Kipf \cite{kipf2016GCN}. % as $ S = D(\hat A)^{-\frac12}\hat A D(\hat A)^{-\frac12}$ with  $\hat{A} = A+I_N$.
%\AR{comme GNN discuté après, enlever cette phrase et éventuellement mettre la réf à un article (le tien?) pour le self loop normalized Laplacian}

%All this classic theory on spectral graph coarsening relies on the combinatorial Laplacian, but before searching for better reduction matrix $P$,  we can extend this theory to a broader class of Laplacian matrix : $\diagmatrix$-Laplacian. We denote  by $\Lnorm  \in \mathbb{R}^{N\times N} $ a broader notion of a Laplacian matrix:
%with $\diagmatrix \in  \mathbb{R}^{N\times N} $ any stricly positive diagonal matrix that depends on the adjacency matrix $ \diagmatrix= f(A)$. This extended definition allows us to describe the combinatorial Laplacian with $\diagmatrix = I_N$, the normalized Laplacian with $\diagmatrix = {D}^{-1/2}$ or the self loop normalized Laplacian with $\diagmatrix  = {(D + I_N)}^{-1/2}$. The last one defines a Laplacian such that $\Lnorm = I_N -S$ with $S$ being the propagation matrix of the classical Graph convolution network defined by Kipf \cite{kipf2016GCN} as $ S = D(\hat A)^{-\frac12}\hat A D(\hat A)^{-\frac12}$ with  $\hat{A} = A+I_N$.

%\AR{Voici une nouvelle version (assez modifiée). Le paragraphe originel est en dessous.
To extend the definition of the RSA constant \eqref{eq:rsa} to the case of the generalized $\diagmatrix$-Laplacian, it is first necessary to ensure that the norms used in the original and coarsened graphs, $G$ and $G_c$, are comparable \cite[Corollary 12]{loukas2019graph}. This requires establishing the consistency of the $\diagmatrix$-Laplacian matrices of $G$ and $G_c$.
This is shown with the following lemma, provided that, starting from any binary well-partitioned lifting matrix $\Qcombi \in \mathbb{R}^{N \times n}$, a generalized lifting matrix is constructed as  
\begin{equation} \label{eq:qnorm}
    \Qnorm = \Qnorm(A, \Qcombi) := \diagmatrix^{-1} \Qcombi \diagmatrix_c.
\end{equation}
where $\diagmatrix_c = \diagmatrix(A_c)$ with $A_c$ defined in \eqref{eq:ac}. Note that $\Qnorm$ is also well-partitioned when $\Qcombi$ is well-partitioned, however it is generally not binary. Instead, the constraint still lies on $\Qcombi$, as shown below. % is then the adaptation of Lem.~\ref{lem:loukas} to the normalized case, which show that the same constraints must apply on $\Qcombi$.
\begin{lemma}[Consistency, adaptation \cite{loukas2019graph}]
\label{lem:consistency}
%Consider a graph $G$ 
%and its corresponding adjacency matrix $A$\ar{pas besoinde A} 
%and the corresponding combinatorial Laplacian $\Lcombi$.  Let's consider a graph coarsening and the corresponding mapping which defines a well-partitioned lifting matrix $\Qcombi \in \mathbb{R}^{N \times n}$. The following properties are equivalent:
Let $\Qcombi$ be a well-partitioned lifting matrix. The two following properties are equivalent:
\begin{enumerate}[label=\alph*)]
    \item $\Qcombi$ is proportional to a binary matrix.
    \item For all adjacency matrices $A$, we have $L(A_c) = \Qnorm(A,\Qcombi)^\top L(A) \Qnorm(A, \Qcombi)$, where we recall that $L$ is defined in \eqref{eq:delta_l} and $\Qnorm$ in \eqref{eq:qnorm}.
\end{enumerate}
%When  $\Qcombi$ is an \emph{$\alpha$-valued} matrix, $\Qnorm$ is said to be "degree-wise" valued.
\end{lemma}
%METTRE LE LEMME (changer binary).\\
%AJOUTER DES PROPRIETES DU lifting généralisé\\
Hence, the normalized Laplacian of the coarsened graph can again be directly deduced from the normalized Laplacian of the original graph when adopting the generalized lifting matrix $\Qnorm = \Qnorm(A,\Qcombi)$ with a well-partitioned and binary $\Qcombi$.
This consistency then enables the definition of a generalized RSA constant 
\begin{equation}\label{eq:rsanorm}
    \epsilon_{\Lnorm,\Qnorm,P,\preservedspace} = \sup_{x\in \preservedspace, \|x\|_\Lnorm =1} \lVert x - \Qnorm P x  \rVert_{\Lnorm}
\end{equation}
where, again, we emphasize that $L$ is defined in \eqref{eq:delta_l} and $\Qnorm$ in \eqref{eq:qnorm}. The main question that we will examine in the rest of the paper remains: for a fixed $\Qcombi$ that is well-partitioned and binary, what are the ``valid'' reduction matrices $P$, and is there a more ``optimal'' choice to minimize $\RSAconst$?

\section{Expanding the space of reduction matrices}
\label{sec:pset}

In this section, we examine the first of the above questions: %considering a well-partitionned and degre-wise valued $Q$ lifting matrix, 
what are the ``valid'' reduction matrices $P$? One might be tempted to simply minimize the RSA equation \eqref{eq:rsanorm} over all  $P$ matrices, however this discards important properties of graph coarsening that we want to preserve. To address this, we introduce several ensembles of admissible $P$ matrices and derive key properties that will enable us, in Sec.~\ref{sec:pexamples}, to optimize $P$ and to compare several examples. 

In the beginning of this section, $Q$ will indicate any well-partitioned matrix. We note that, unlike all the matrices considered in the literature, \emph{some matrices $P$ with a support different from $Q^\top $ will be acceptable}. Intuitively, $\Qnorm$ defines the ``true'' mapping between the nodes and the supernodes by enforcing the well-partitioned aspect, while $P$ will be allowed to relax this constraint.
We begin by examining the largest such ensemble, denoted $E_1$.

%
%As we have seen, the important object in assessing the RSA constant is actually the \emph{coarsen and lift operator} $\Pi = QP$. %Indeed, for a \emph{signal} $x \in \mathbb{R}^N$ defined on \(G\) nodes, the goal is that the coarsened signal $x_c=Px$ then reconstructed signal $\tilde{x} = Qx_c = \Pi x$ loose the less ``informations'' which is possible comparing with the original signal. As we have seen in the previous section, it comes down to an optimization problem looking for preserving the low frequencies of a Laplacian(eq.~\ref{eq:rsa}).
%
%This gives a initial lead to characterize the matrices $P$.

\paragraph{$E_1$: $P$ such that $\Pi$ is a projection.}
A first minimal property in graph coarsening is that \emph{applying successively coarsen and lift procedures does not degrade further the signal}, or in other words, the coarsen and lift operator $\Pi = QP$ is a projection $\Pi^2 = \Pi$. 

The classical choice, where $P$ is the Moore Penrose inverse of $Q$, satisfies the projection property -- $\Pi$ is even an \emph{orthogonal} projector in this case. However it is only an example among a bigger set of reduction matrices which satisfies $\Pi^2=\Pi$. To characterize this larger set, we first recall the notion of \emph{generalized inverse}.

%\begin{minipage}{0.5\textwidth}

\begin{definition}[Generalized Inverse]\label{def:inv}
Let \( A \in \mathbb{K}^{m \times n} \). We consider a matrix $B \in \mathbb{R}^{n \times m}$ which can satisfy the following conditions:

\begin{enumerate}[label=(\roman*)]
    \item \label{cond1} \( B \in A^g \quad with \quad  A^g := \{ M \, | \, AMA = A \} \)
    \item \label{cond2} \( A \in B^g \quad i.e \quad  BAB = B\)
    \item \label{cond3} \( AB \) is Hermitian: \( \left(AB\right)^* = AB \).
    \item \label{cond4} \( BA \) is Hermitian: \( \left(BA\right)^* = BA \).
\end{enumerate}
The matrix B is said to be:
\begin{itemize}[label=$\bullet$]
    \item \textbf{generalized inverse} of $A$ when it satisfies ~\ref{cond1};
    \item \textbf{generalized reflexive inverse} of $A$ if it satisfies simultaneously~\ref{cond1} and \ref{cond2};
    \item \textbf{Moore Penrose inverse} of $A$ if it satisfies ~\ref{cond1}, \ref{cond2}, \ref{cond3} and \ref{cond4}.
\end{itemize}
\end{definition}
%\end{minipage}
%\hfill
%\begin{minipage}{0.45\textwidth}

%\end{minipage}

\begin{comment}

\begin{definition}[Generalized Inverse]\label{def:inv}
Let \( A \in \mathbb{K}^{m \times n} \). We consider a matrix $B \in \mathbb{R}^{n \times m}$ which can satisfy the following conditions :

\begin{enumerate}[label=(\roman*)]
    \item \label{cond1} \( B \in A^g \quad with \quad  A^g = \{ M \, | \, AMA = A \} \)
    \item \label{cond2} \( A \in B^g \quad i.e \quad  BAB = B\)
    \item \label{cond3} \( AB \) is Hermitian: \( \left(AB\right)^* = AB \).
    \item \label{cond4} \( BA \) is Hermitian: \( \left(BA\right)^* = BA \).
\end{enumerate}
The matrix B is said to be:
\begin{itemize}[label=$\bullet$]
    \item \textbf{generalized inverse} of $A$ when it satisfies ~\ref{cond1} ;
    \item \textbf{generalized reflexive inverse} of $A$ if it satisfies simultaneously~\ref{cond1} and \ref{cond2} ;
    \item \textbf{Moore Penrose inverse} of $A$ if it satisfies ~\ref{cond1}, \ref{cond2}, \ref{cond3} and \ref{cond4}.
\end{itemize}
\end{definition}

\end{comment}
\begin{remark}[uniqueness]
The Moore Penrose inverse is \textbf{unique} while there may exist \textbf{infinitely} many "generalized inverses" \cite{penrose1955generalized}. %moore1920reciprocal si place
\end{remark}

\begin{wrapfigure}{r}{0.45\textwidth}
\vspace{-25pt}
\centering
\resizebox{.42\textwidth}{!}{
% Définition des styles de nœuds pour les croix de différentes couleurs
\tikzset{
  crossred/.style={
    path picture={
      \draw[thick, red]
        (path picture bounding box.south west) -- (path picture bounding box.north east)
        (path picture bounding box.south east) -- (path picture bounding box.north west);
    },
    inner sep=0pt,
    minimum size=6pt
  },
  crossblue/.style={
    path picture={
      \draw[thick, blue]
        (path picture bounding box.south west) -- (path picture bounding box.north east)
        (path picture bounding box.south east) -- (path picture bounding box.north west);
    },
    inner sep=0pt,
    minimum size=6pt
  },
  crossgreen/.style={
    path picture={
      \draw[thick, green]
        (path picture bounding box.south west) -- (path picture bounding box.north east)
        (path picture bounding box.south east) -- (path picture bounding box.north west);
    },
    inner sep=0pt,
    minimum size=6pt
  }
}

\begin{tikzpicture}

% Première ellipse (A)
\filldraw[fill=gray!20, draw=black] (0,0) ellipse (3.5 and 2);
\node at (0, 1.6) {$E_1 = \{P~|~Q \in P^g \}$};

% Seconde ellipse (B), à l'intérieur de A, décalée à droite
\filldraw[fill=gray!40, draw=black] (0.8,-0.1) ellipse (2 and 1.5);
\node at (0.8, 1.0) {$E_2 = Q^g$};

% Troisième ellipse (C), à l'intérieur de B, décalée vers le bas
\filldraw[fill=gray!60, draw=black] (0.7,-0.35) ellipse (1.6 and 1.1);
\node at (0.8, 0.1) {$E_3 = Q^g_{supp}$};
% \shortstack{$\operatorname{supp}(P)$ \\ $= \operatorname{supp}(Q^T)$}
%{$\operatorname{supp}(P) = \operatorname{supp}(Q^T)$};
%{$\shortstack{\operatorname{supp}(M) \\ = \operatorname{supp}(L)}$}
%{\parbox{2cm}{\centering $\operatorname{supp}(M)$ \\ $= \operatorname{supp}(L)$}}
%{$\begin{aligned} \operatorname{supp}(M) \\ =\operatorname{supp}(L) \end{aligned}$};

% Point "Rat" dans la première ellipse
\node[crossred, thick] at (-2, 0.7) {};
\node at (-2, 0.3) {$P_{opt}$};

% Point "L" dans la troisième ellipse
\node[crossblue, thick] at (0.2, -0.6) {};
\node at (0.2, -1.0) {$P_{Loukas}$};

% Point "MP" dans la troisième ellipse
\node[crossgreen, thick] at (1.3, -0.6) {};
\node at (1.3, -1.0) {$P_{MP}$};

\end{tikzpicture}
}
\caption{\small Ensembles of admissible reduction matrices.
$E_1$ includes all $P$ satisfying only the projection constraint for 
$PQ=\Pi$. $E_2$ contains all generalized inverses of $Q$, and is shown to be a subset of $E_1$ in Lem.~\ref{lem::genrefinv}.
$E_3$ restricts $E_2$ to $P$ matrices sharing the same support as $Q^T$.}
%Scheme representing the set of reduction matrices $P$ such that $\Pi$ be a projection for a well-partitioned valued lifting matrix $Q$ with $\Qcombi$ binary, defined as $\{P~|~Q \in P^g \}$ in Lem.~\ref{lem:piproj}.} %$Q^g$ is characterized by Lem.~\ref{lem:caracginv} and $E_3$ by Lem.~\ref{lem:caracsupp}
\label{fig:schema}
\vspace{-5pt}
\end{wrapfigure}
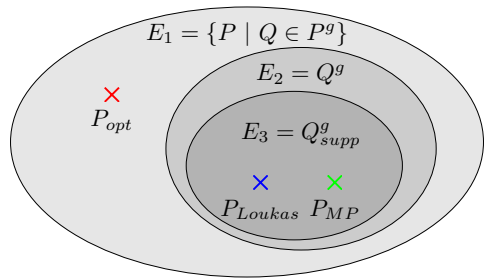
We now propose an alternative characterization of the $E_1$ ensemble.
%For a well-partitioned lifting matrix $Q$, we have the following result.
\begin{lemma}[Generalized Inverse and $\Pi$ projection]\label{lem:piproj}
For a well-partitioned lifting matrix $Q$:
\[
 \Pi^2=\Pi  \quad \Longleftrightarrow
 \quad Q \in P^g
\]
\end{lemma}
The proof can be found in App.~\ref{app:proofpiproj}. Lem.~\ref{lem:piproj} means that, \emph{assuming that $Q$ is well-partitioned}, $E_1 = \{ P~|~\Pi \text{ projection} \}=\{ P~|~Q \in P^g \}$. Note that the hypothesis on $Q$ is important here: this lemma is not true in general, and only valid because $Q$ is particularly simple.
%By assuming $\Pi$ projection for a well-partitioned matrix $Q$, the set $E_1 = \{ P~|~Q \in P^g \}$ defines the biggest set of admissible reduction matrix $P$ as shown in Fig.~\ref{fig:schema}.

In general, the set $E_1$ does not admit further description, and
it seems difficult to optimize over with algorithms such as projected gradient descent. We see next that a relatively minor additional constraints considerably simplifies the situation.

\paragraph{$E_2$: $P$ generalized inverse of $Q$.} We now consider an ensemble characterized in a reverse manner of $E_1$. The rationale is that this set admits a closed form characterization, which allows for an easy implementation of optimization algorithms over it.
%Since $Q$ needs to be in $P^g$, a natural additional constraint is to add the reverse. As it turns out, it is far easier to characterize the generalized \emph{reflexive} inverse of a well-partitioned $Q$. 
\begin{lemma}[Generalized reflexive inverse]
\label{lem::genrefinv}
For a well-partitioned lifting matrix $Q$ and a reduction matrix $P$ such that $Q \in P^g$, we have the following equivalence:
\[
\operatorname{rank}(P) = n \quad \Longleftrightarrow
 \quad P \in Q^g
\]
Conversely, $P \in Q^g$ implies $Q \in P^g$ and $\operatorname{rank}(P) = n$, such that $Q^g \subset E_1$. %leading to the inclusion of this set in the previous one (see Fig.~\ref{fig:schema}).
\end{lemma}
Again, this proof relates specifically to well-partitioned matrices $Q$. We thus define $E_2 := Q^g$, and the lemma shows that $E_2 \subset E_1$ (Fig.~\ref{fig:schema}). The inclusion is often strict, as there generally exists $P$ such that $Q \in P^g$ but such that $\operatorname{rank}(P)<n$. The proof of this lemma can be found in App.~\ref{app:prfreflexinv}.
\begin{remark}[Generalized inverse $\Rightarrow$ reflexive] In Lem.~\ref{lem::genrefinv} we show that, for well-partitioned $Q$, generalized inverses of $Q$ are automatically \emph{reflexive} generalized inverses. Of course, this is not true in general, here this is again due to the fact that $Q$ is well-partitioned.
\end{remark}
%Unlike most matrices, for a well-partitioned $Q$, its generalized inverses are also reflexive. 
%Interestingly, the \emph{reflexive} generalized inverses of $Q$ are far easier to characterize than the $E_1$ ensemble.
As hinted above, the set $E_2$ is far easier to describe than $E_1$.
\begin{lemma}[Characterization of generalized reflexive inverses of $Q$]
    \label{lem:caracginv}
    Let $Q \in \mathbb{R}^{N \times n}$ be a well-partitioned lifting matrix. All the reflexive generalized inverses of $Q$ can be characterized as: 
    \[
    E_2 = Q^g =\{ Q^+  + M\,(I_N - QQ^+) ~|~ M \in \mathbb{R}^{n \times N}\}
    \]
    with $Q^+$ the Moore Penrose inverse of $Q$.
\end{lemma}
These two lemmas are important because they provide a way to optimize $P$. Indeed, by the converse of Lem.~\ref{lem::genrefinv}, it is shown that $E_2$ (being a subset of $E_1$) contains admissible matrices. Moreover, Lem.~\ref{lem:caracginv} (proven in App.~\ref{app:caracginv}) offers a convenient characterization of $E_2$ through the matrix $M$. 

%Thanks to this property, proven in App.~\ref{app:caracginv}, we may now consider computing an ``optimal'' $P$ for the RSA by optimizing directly the matrix $M$. Moreover, we note that unlike all the matrices considered in the literature, \emph{some matrices with a support different from $Q^\top $ are acceptable.} Intuitively, $\Qnorm$ defines the ``true'' mapping between the nodes and the supernodes by enforcing the well-partitioned aspect, while $P$ is allowed to relax this constraint. 

\paragraph{$E_3$: $P$ generalized reflexive inverse with same support}
Up until now, matrices in $E_1$ and $E_2$ have no reason to have the same support as $Q^+$, unlike for instance the Moore-Penrose inverse. Worse, matrices in $E_2$ may be very dense, which might hinder computation time and increase memory usage. Therefore, we consider sparser $P$, and add the constraint that $P$ has the same support as $Q^\top$, while still being a reflexive inverse. By construction, $E_3 \subset E_2$, as shown in Fig.~\ref{fig:schema}. Moreover, $E_3$ is not empty, as it contains at least the Moore-Penrose inverse (see next section).
\begin{lemma}[Generalized reflexive inverse with same support]
\label{lem:caracsupp}
Let $Q = \Qnorm(A,\Qcombi) \in \mathbb{R}^{N \times n}$ be generalized lifting matrix with $\Qcombi$ well-partitioned and binary. %Let $\alpha \in \mathbb{R}$ be the non zero value of the associated binary lifting matrix $\Qcombi$. 
The set of reflexive generalized inverse of $\Qnorm$ with the same support as $\Qnorm^\top$ is defined as : 
\[
    E_3  =\left\{ P \in \mathbb{R}^{n\times N} \;\middle|\;
    \begin{cases}
        supp(P) = supp(Q^\top ) \\
         \sum_{k = 1}^N \frac{P_{ik}}{\diagmatrix(k)}  = \frac{1}{\diagmatrix_c(i)} \quad \forall i \in [1,n] 
    \end{cases}
    \right\}
\]
\end{lemma}
%\NK{en fait j'ai l'impression qu'on utilise $Q^{g}_{supp}$ nulle part. A supprimer?}
Note that, while Lem.~\ref{lem:piproj}, \ref{lem::genrefinv} and \ref{lem:caracginv} were valid for \emph{any} well-partitioned matrix $\Qnorm$, here we specifically examine $Q = \Qnorm(A,\Qcombi)$ with $\Qcombi$ well-partitioned \emph{and} binary.
%\begin{remark}[Combinatorial Laplacian]
%  For the combinatorial Laplacian ($\diagmatrix = I_n $), the generalized inverses with same support as $\Qcombi^\top$ are the reduction matrices whose rows sum to $1$. \NK{remarque pas forcément nécessaire, plutôt évident}
%\end{remark}
%When $\diagmatrix = D^{-1/2}$ (degree normalization), the generalized inverse of same support consists in all the matrices with the same support as $Q^\top$ with the degree-weighted sum of the weight of the nodes in the same super-node is equal to the square root of the degree of the super-node in the coarsened graph. The liberty degree becomes the weight of each node in their supernode before being re-normalized.\AJ{ici à reformuler et pas essentiel}\\
%
Moreover, all matrices with non-zero coefficients on the support of $\Qnorm^\top$ that are also in $E_1$ are in $E_2$ and $E_3$, as their rank is equal to $n$. %\AR{mettre cette remarque que si tu as la place} 
%\AJ{
To provide a better intuition about Lem.~\ref{lem:caracsupp}, consider the case of the combinatorial Laplacian ($\diagmatrix = I_N$) where $Q = \Qcombi$; this lemma reduces to $\sum_{k = 1}^N P_{ik} = 1$, which appears to be a natural condition already used in \cite{bianchi2023expressive}.
%\AR{such as XX XX}. in a concurrent litterature

The lemma is proved in App.~\ref{app:caracsupp}. 
Note that %it enables faster optimization than the one using Lem.~\ref{lem:caracginv}, 
optimizing over $E_3$ is particularly light compared to the previous dense examples, as it requires optimizing only over the $N$ non-zero coefficients located on the support of $Q^\top$. Projected gradient descent can be implemented with a simple renormalization of the rows at each iteration.

\section{From classical to novel reduction matrices: a comparative study}
\label{sec:pexamples}

Now that we have proposed a taxonomy of the valid reduction matrices $P$, we investigate the second question raised in Sec.~\ref{sec:rappel}: what are good examples of reduction matrices? In all this section, we consider $Q = \Qnorm(A,\Qcombi)$ with $\Qcombi$ well-partitioned and binary. As outlined in Sec.~\ref{sec:rappel}, $L_c = Q^\top L Q$ is then the $\Delta$-Laplacian of the coarsened graph. We will start with three examples with closed-form analytic expression, then outline a possible optimization framework, emphasizing that it is only one choice among many possible. It is worth noting that the transposed matrix $P= Q^\top$, commonly used in graph pooling, does not belong to $E_1$, and is therefore not considered in our analysis.

%In particular, 
%We study several cases of matrix $P$, which we place within the framework of the previous section (Fig.~\ref{fig:schema}).  
 
%
%For a fixed $Q$ matrix, usually in the literature \cite{loukas2019graph, dickens2024graph, kumar2023featured}, the most common choice is to consider the reduction matrix $P$ as the Moore Penrose inverse of $Q$ to minimize the RSA.
%As we have previously shown, there is a broader class of "admissible" reduction matrix. We will then study examples of $P$ that we will place on our Scheme Fig.~\ref{fig:schema}. 

\paragraph{Moore-Penrose Reduction.}\label{para:MorePen} 
The most common choice for $P$ in the literature \cite{loukas2019graph, dickens2024graph, kumar2023featured}, is simply the Moore Penrose inverse of $Q$ (the derivation is in App.~\ref{app:proofmp}):
\begin{equation}
\label{eq::pmp}
P_{MP} := Q^+ = (Q^\top Q)^{-1}Q^\top 
\end{equation}
By Def.~\ref{def:inv}, $P_{MP} \in E_2$ and  has the same support as $Q^\top $ since $Q^\top Q$ is diagonal when $\Qcombi$ is well-partitioned. Thus, $P_{MP} \in E_3$.
%\AJ{In the case of the combinatorial Laplacian $\Lcombi$, for a well-partitioned $1$-valued lifting matrix $\Qnorm = \Qcombi$, the Moore Penrose inverse is the intuitive reduction matrix that average uniformly the signals on each supernodes of $G_c$.}
%\AJ{In the case of the combinatorial Laplacian $\Lcombi$, for a well-partitioned $\alpha$-valued $\Qnorm = \Qcombi$ and $\alpha$-valued. the moore penrose inverse of $Q$ is the matrix respecting the condition of lemma Lem.~\ref{lem:caracsupp}, such that all the weights of the same super-node are constant and equal to $\frac{1}{\alpha n_k}$ where $n_k$ is the number of non zero coefficient of $Q$ on column k (the number of nodes of $G$ gathered in the super-node $k$ of $G_c$). It is the most common setting applied in the literature, the features of the super-node being then the average of the features in the original nodes.}

Interestingly, $P_{MP}$ is the solution of the following optimization problem, which we note is formulated over \emph{all} matrices $P\in\mathbb{R}^{n \times N}$ without any constraint:
\begin{equation}\label{eq:optimMP}%\epsilon= \epsilon(G,Q,\preservedspace, \Lcombi) 
        \arg\min_P \sup_{x\in \mathbb{R}^N, \|x\|_2 =1} \lVert x - QP x  \rVert_{2}.
\end{equation}
This problem looks suspiciously similar to the RSA \eqref{eq:rsanorm}, but differs in several key aspects. Namely, the signals lives in $\mathbb{R}^N$ not in $\preservedspace$, the Mahalanobis norm $\|\cdot\|_{L}$ in \eqref{eq:rsanorm} is replaced by the Euclidean $l_2$ norm. This suggests that $P_{MP}$ is still ``optimal'' from a certain point of view, but for a different (much simpler) problem than the RSA. Below, we will formulate an optimal matrix for a problem closer to the RSA, but first examine another potential choice in $E_3$.
%\emph{even for a fixed $Q$, it is possible to find a $P$ resulting in a better RSA}.

\paragraph{Loukas Reduction \cite{loukas2019graph}, aka iterative coarsening.} \label{para:loukasreduc}
Many coarsening algorithms construct the lifting matrix iteratively, as a product of individual coarsening $Q = Q_1 \ldots Q_c$. In \cite{loukas2019graph}, Loukas implements such an algorithm, and chooses the reduction matrix as the product of the Moore-Penrose inverses of each lifting matrix
\begin{equation}\label{eq:ploukas}
\Ploukas = Q_c^+ \ldots Q_1^+
\end{equation}
Note that this is \emph{not} equal to the Moore-Penrose inverse of $Q$ in general. However, and as shown in App.~\ref{app:ploukas}, this results in a matrix $\Ploukas \in E_3$, (see Fig.~\ref{fig:schema}).%, since it is a generalized inverse of $Q$ and has the same support as $Q^\top.$

%Most of coarsening construction proceeds in an hierarchic way, merging iteratively nodes.This leads to a lifting matrix $Q = Q_1 \ldots Q_c$. Loukas \cite{loukas2019graph} decided to consider as reduction matrix the product of the intermediary lifting matrix Moore-Penrose inverse:
%Please notice that this matrix \emph{is not} necessarily equal to the Moore Penrose inverse of $Q$. However it is still a reflexive generalized inverse of same support of $\Qnorm$ as proven in App.~\ref{app:ploukas}.

\paragraph{Rao and Mitra inspired reduction} %In search for matrices with lower RSA constant, we now consider a relaxation of the RSA constant minimization \eqref{eq:rsa}, where, in comparison to 
As we have seen above, the Moore-Penrose inverse can be interpreted as the solution of an optimization problem \eqref{eq:optimMP}, that differs from the RSA \eqref{eq:rsanorm} in two key aspects: the subspace $\mathcal{R}$ and the Mahalanobis norm. We now consider the following problem, where we reintroduce the latter:
%the norm is the one used for the RSA constant:
%The NP-hard nature of the RSA optimization problem motivates the relaxation of the restriction of the $\preservedspace$ to the global $\mathbb{R}^N$. This leads us to a new formulation of this optimization problem:
\begin{equation}\label{eq:optimRAO}
        \arg\min_P \sup_{x\in \mathbb{R}^N, \|x\|_L =1} \lVert x - QP x  \rVert_{\Lnorm}
\end{equation}
Rao and Mitra show that \eqref{eq:optimRAO} admits a unique solution under some hypotheses. In our case, several simplification over their original result happen, and we obtain the following: if \textbf{$L$ and $L_c$ are positive definite}, then the optimal solution is $P = L_c^{-1}Q^\top L$. This solution does not technically apply when $L$ and $L_c$ are $\Delta$-Laplacians since they are not invertible, 
%In other terms we have \emph{reintroduced} the Mahalanobis norm in \eqref{eq:optimMP}. \textbf{When $L$ and $L_c= Q^\top LQ$ are definite positives}, this problem admit an unique solution of minimal norm proposed by Rao in the 1970s \cite{rao1972generalized}, which is equal to  $L_c^{-1}Q^\top L$. Although this problem is closer to the RSA compared to \eqref{eq:optimMP}, two differences subsist : 1) There is still no mention of the space $\preservedspace$, and 2) $L$ and $L_c$ are semi definite positive and not inversible.
%However, 
but inspired by this, we propose the following reduction matrix:
\begin{equation}\label{eq:prao}
\Prao = L_c^+ Q^\top  L
\end{equation}
Note that, even though \eqref{eq:optimRAO} is again an optimization problem with no constraints on $P$, %$\Prao$ is an admissible solution since $Q \in \Prao^g$, i.e. 
it is easy to check that $\Prao\in E_1$. However $\Prao \notin E_2$ as it is not full rank (see Fig.~\ref{fig:schema}).
%Moreover, $Q \in \Prao^g$ i.e. $\Prao\in E_2$, but $\Prao \notin E_3=Q^g$ as it is not full rank (see Fig.~\ref{fig:schema}). 
%The rationale for $\Prao$ is that 
Hopefully, $\Prao$ should lead to a better RSA constant than $P_{MP}$, even though $\mathcal{R}$ is still absent from \eqref{eq:optimRAO}. %since it results from a minimization using the same norm as the RSA constant. 
However, its main drawback is that it is dense in general. %, and even full support. %This motivates a formulation that promotes sparsity in the coefficients.
%Oppositely to $P_{MP}$ and $\Ploukas$,  $\Prao$ does not have the same support as $Q^\top $ and it is actually entirely dense with  $O(n\times N)$ non zero terms. This motivates an algorithmic search of the optimization problem potentially with additional constraints. 

\paragraph{Optimization based Reduction} 
%\AJ{Je pense faut pas amener ce dernier argument car casse tete meme loukas n'y va pas : This is a relevant because the RSA has a close form $\epsilon = \lVert \Lnorm^{1/2} (I_N - \Pi) V V^\top  \Lnorm^{-1/2} \rVert $ with $V$ the basis of $\preservedspace$ }
%We have seen previously that having $P$ such that $\Pi = QP$ is preferable. At our knowledge there is not an easy characterization of $\{ P | Q \in P^g\}$. Nonetheless the smaller subset $Q^g$ is easier to characterize with Lem.~\ref{lem:caracginv} and thus perform optimization on. We then have the following optimization problem :
We now turn to the true RSA minimization of \eqref{eq:rsanorm}. To our knowledge, it does not have a simple solution such as $P_{MP}$ for \eqref{eq:optimMP} or $\Prao$ for \eqref{eq:optimRAO}, so that we need to implement an iterative optimization algorithm. As discussed in the previous section, it is particularly convenient to minimize over the set $E_2$, thanks to the characterisation of $E_2$ with Lem.~\ref{lem:caracginv}. The minimization can then be written as:
\begin{equation}\label{eq:optimRSAginv}
        P^*_{g} = \Phi_Q(M^*) \quad \text{with} \quad M^* = \arg\min_{M\in \mathbb{R}^{n\times N}} \sup_{x\in \preservedspace, \|x\|_L =1} \lVert x - Q\Phi_Q(M)x  \rVert_{\Lnorm}
\end{equation}
with $\Phi_Q(M) = Q^+  + M\,(I_N - QQ^+)$.
%\AJ{Although it is a non-convex problem, in the experiments (Sec.~\ref{sec:experiments}), optimizing the RSA value in the the subspace of generalized reflexive inverse converges to a value close to the RSA associated to $\Prao$ which confirm the relevance to consider it.}
Again, solutions $P^*_{g}$ are usually dense. A potential remedy is to add a sparsity constraint on $P$, which leads to the following problem:
%$L1$ regularisation and then a threshold to enforce the sparsity of the results. This is the second variant of our optimization problem with $\lambda$ being an hyper-parameter. 
\begin{equation}\label{eq:optimRSAginvsparse}
        P^*_{g,l_1} = \Phi_Q(M^*_{l_1}) \quad \text{with} \quad M^*_{l_1} = \arg\min_M \sup_{x\in \preservedspace, \|x\|_L =1} \lVert x - Q\Phi_Q(M) x  \rVert_{\Lnorm} + \lambda \lVert \Phi_Q(M) \rVert_1
\end{equation} 
%\NK{both min-max problems and $\ell_1$ minimization are complicated in general. A word about how you handle that?}
As mentioned before, when $\preservedspace$ is a subspace of $\Lnorm$, the RSA has a close form expression that is convex in $P$. This results in optimization problem that are convex in $M$ in these cases. In our experiments, we treat the $l_1$ penalty simply with gradient descent combined with a final thresholding operation (for simplicity we leave aside more complex optimization procedures with e.g. proximal operators).

%In the case where the graph is large, \eqref{eq:optimRSAginvsparse} can become complex. 
Finally, we also mentioned that optimization over $E_3$ was also particularly simple: a simple renormalization is sufficient for projected gradient descent. It has the advantage of being always sparse, as it respects the %support of $Q^\top$ %We therefore propose a simplified optimization problem, where the minimization is performed only over the 
support of $Q^\top$ (of size $N)$. %, rather than over all coefficients of $P$ (i.e. $n\times N$ terms). 
This leads to the following problem:
%The search space of the matrix $M$ is $\mathbb{R}^{n\times N}$ and thus really big. Consequently, we can perform only this optimization for small graphs up to thousands of nodes such Cora and Citeseer. This motivates the idea to make an optimization problem on the same support as $Q^\top$ characterized in Lem.~\ref{lem:caracsupp}. 
\begin{equation}\label{eq:optimRSAsupp}
        P^*_{Q^\top} = \arg\min_{P \in E_3} \sup_{x\in \preservedspace, \|x\|_L =1} \lVert x - QP x  \rVert_{\Lnorm}
\end{equation} 
This is again convex in $P$, as the constraint $P\in E_3$ is linear.
%In practice, optimization is carried out over the $N$ non zero terms of $P$, and the solution is projected onto the subspace $E_3$ characterized in Lem.~\ref{lem:caracsupp}. 

%The results of these different optimization problems are discussed in the following experiments section. 

\section{Experiments}
\label{sec:experiments}
In Sec.~\ref{sec:pexamples}, we have proposed several examples of reduction matrices $P$ that aimed to minimize the RSA with various constraints.
In this section, we evaluate numerically the performance of these examples, both in terms of RSA constant and used within GNNs.
%values of the obtained constants, compare them to various existing approaches, and finally apply the constructed matrices in the context of GNNs. We begin by specifying the graphs, Laplacians, and algorithms used to construct the lifting matrix $Q$.

\paragraph{Setup:} \textbf{(i) Graph.} We consider the two classical medium-scale graphs Cora \cite{mccallum2000automating}, and Citeseer \cite{giles1998citeseer}, %and a larger Graph Reddit \cite{hamilton2017inductive}. 
and use the public split from \cite{yang2016revisiting} for training the GNNs. We restrict ourselves to medium-scale graphs because handling larger graphs presents challenges. Indeed, in the optimizations we propose, the RSA requires computing the square root of the original Laplacian, which is not sparse in general and, for graphs like Reddit \cite{hamilton2017inductive}, cannot be stored on modern GPUs 
%\AJ{Dire que c'est plutot le score qui pose problème et qui a besoin de la racine carré du graphe \emph{original}, l'optimisation est scalable pas le score RSA}
. Furthermore, we only consider the largest connected component since connected graphs are better suited for coarsening (see details in App.~\ref{app::datasets}). This may induce some slight difference with other reported results on these datasets.

\textbf{(ii) Laplacian and preserved space.} We choose two different Laplacians: the combinatorial Laplacian $\Lcombi$ and the self-loop normalized Laplacian $\Lnorm = \diagmatrix\Lcombi\diagmatrix$  with $\diagmatrix = (\operatorname{diag}(A1_N) +1)^{-1/2}$. The motivation for this Laplacian comes from the fact that it is related to the propagation matrix $S =  \hat{D}^{-1/2}(A+I_N) \hat{D}^{-1/2} $ (via $\Lnorm = I_N - S$) commonly used in GNNs \cite{kipf2016GCN,joly2024graph}. For the RSA, the preserved space $\preservedspace$ is chosen as the $K=100$ first eigenvectors of $\Lcombi$ and $\Lnorm$. % ($K=100$ in our experiments).

\textbf{(iii) Lifting matrix $\Qnorm$.} The method that has introduced the notion of RSA \cite{loukas2019graph} is particularly relevant for computing $\Qnorm$, as we share the same objective of reducing the RSA constant. However, it requires adaptation because this method minimizes the RSA constant for the combinatorial Laplacian, and  yields a binary well-partitioned matrix lifting matrix $\Qcombi$. We generalize this method to minimize the RSA constant for a  $\Delta$-Laplacian, resulting in a lifting matrix $Q$. It is worth noting that some of the assumptions made in \cite{loukas2019graph} no longer hold in the $\Delta$-Laplacian setting. Nevertheless, we observe that this generalized construction still achieves good RSA constants.
The details of the generalized algorithm are provided in App.~\ref{app:loukasalg}.

\paragraph{RSA minimization.}  The RSA constants achieved by the reduction matrices introduced in Sec. ~\ref{sec:pexamples} are shown in Fig.~\ref{fig:rsa_cora_combi} for the combinatorial Laplacian $\Lcombi$ and in Fig.~\ref{fig:rsa_cora_norm} for the self-loop normalized Laplacian $\Lnorm$. The coarsening ratios range from 0.05 to 0.85. As expected, the two proposed methods $\Prao$ and $P^*_{g}$ achieve better RSA constants than the usual $P_{MP}$ and $\Ploukas$. Moreover, the performance gap increases with higher reduction ratios. Another interesting observation is that the optimization methods $P^*_{g,l_1}$ and $P^*_{Q^\top}$ that incorporate sparsity constraints also perform very well. At low reduction ratios, their performance is nearly indistinguishable from their unconstrained counterparts $P^*_g$. This is particularly surprising given that the reduction in the number of non-zero coefficients is around $99.8\%$
%\AR{vérfie ce chiffre. Tu avais mis 99.5 masi j'ai calculé 99.8?}
, regardless of the coarsening ratios (see App.~\ref{app:nnz} for a detailed analysis on the sparsity of the $P$ matrices and App.~\ref{app:optimparam} for their computational hyperparameters). We also observe that the two approaches $P^*_{g,l_1}$ and $P^*_{Q^\top}$ yield similar performance. This may be due to the strong regularization parameter used in the sparsity penalty in $P^*_{g,l_1}$; a different setting could yield a more balanced trade-off for $P^*_{g,l_1}$ between the performance of $P^*_{g}$ and $P^*_{Q^\top}$.

\begin{figure}[htbp]
    \centering
    \begin{subfigure}[b]{0.49\textwidth}
        \includegraphics[width=\linewidth]{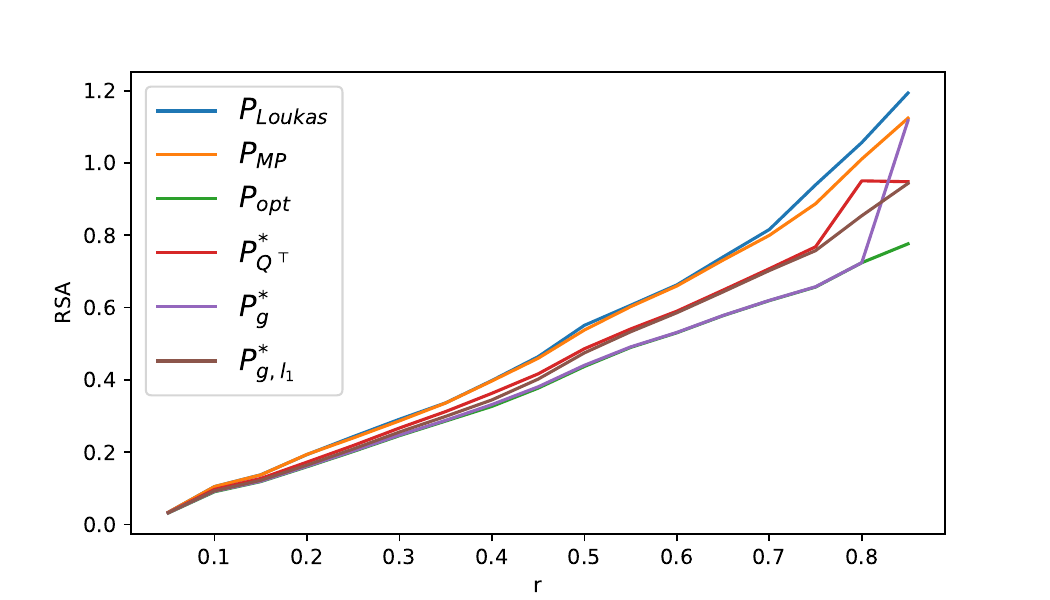}
        \caption{Cora graph, combinatorial Laplacian $\Lcombi$}
        \label{fig:rsa_cora_combi}
    \end{subfigure}
    \hfill
    \begin{subfigure}[b]{0.49\textwidth}
        \includegraphics[width=\linewidth]{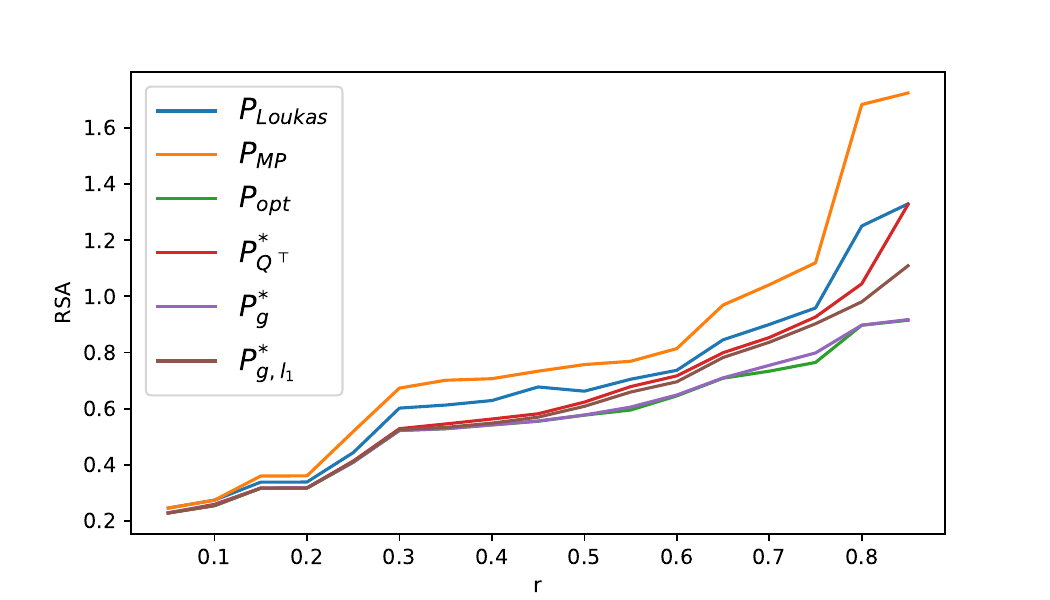}
        \caption{Cora graph, self-loop normalized Laplacian $\Lnorm$}
        \label{fig:rsa_cora_norm}
    \end{subfigure}
    \caption{RSA for different reduction matrices}
    \label{fig:rsa_cora}
\end{figure}

\paragraph{GNN application}  Inspired by the paper \cite{joly2024graph} which link the training of Graph Neural Networks (GNN) and the RSA : we have trained  for three different coarsening ratio ( $r = \{0.3, 0.5, 0.7\}$) a convolutional GNN  \cite{kipf2016GCN} and a Simplified convolution network (SGC \cite{wu2019SGC}) on Cora \cite{mccallum2000automating} and Citeseer \cite{giles1998citeseer}.  %Only the SGC model strictly verifies the assumption of \cite{joly2024graph} to link its performances with the RSA.  
Each training is averaged on $10$ random split, using the same experimental setting as in \cite{joly2024graph}, and the hyperparameters are provided in App.~\ref{app:gcnhyperparameters}.
For the training on coarsened graphs, we used the propagation matrix $S_c^{MP} = PSQ$ from the paper \cite{joly2024graph}, where $S - L$ and $L$ is the self-loop normalized Laplacian. This matrix depends on both the reduction matrix $P$ and the lifting matrix $Q$. %and not only the lifting one as it could be expected. 
%The reduction matrix $P$ is then doubly involved in defining the coarsened features and in the propagation scheme.
%\AR{coupe cette pharse en 2}.
%We use the self-loop normalized Laplacian and  % such that $S$ being $\preservedspace$-preserving, assumption of \cite{joly2024graph}, is verified.
%For downstream tasks we use the "Max acc possible" metric from \cite{joly2024graph} to analyze a specific coarsening. It corresponds to the optimal prediction over the super-nodes of the coarsened graph (all the nodes coarsened in a super nodes has the same prediction, optimally the majority label of this cluster). It might be hard to achieve as the optimal assignment for the validation nodes or training nodes can be different. It mitigates the performance of the GNNs on the coarsened graph by defining an upper-bound on their accuracy.

The results reported in Tab.~\ref{tab:trainingresultsGCN} show %that a better RSA leads to 
slightly better performances on Cora for high coarsening ratio and RSA-optimized reduction matrices such as $P_g^*$, $P_{g,l_1}^*$ and $P_{Q^\top}^*$. The results on CiteSeer are less pronounced, which may be due to its higher level of heterophily compared to Cora: indeed, since spectral coarsening is designed to preserved the low frequencies of the Laplacian and \cite{joly2024graph} %relies on the assumption that the features live in those frequencies, 
the RSA is more relevant when homophily is high. Lastly, we note that $\Prao$ has the best RSA but poor GNN performance. We might explain this by its \emph{density} which implies a propagation with $S_c^{MP}$ that is similar to a complete graph. The sparsity of each matrix can be found in App.~\ref{app:nnz}. Therefore, the RSA only relatively translates to a better GNN accuracy, mitigating the theoretical results of \cite{joly2024graph}. %Other parameters have to be taken into account such as the sparsity of the reduction matrix.

%\AR{à mettre après: As we will see in the analysis of our GNNs experiments, the RSA is not sufficient to explain all the performances of GNN. Finding a more complex and scalable score to optimize, and deal with larger graphs is an important lead for future work.}

%\AR{Limitation pour après: The design of more scalable coarsening algorithms with RSA guarantees is an important path for future work, but out-of-scope of this paper.\AJ{a reformuler car très similaire ancien neurips}}

\begin{table}[ht]
    \vspace{-10pt}
    \centering
    \small
    \caption{Accuracy in $\%$ for node classification with SGC and GCNconv on different coarsening ratio}
    %\smaller
    \vspace{10pt}
    \begin{tabular}{ccccccc}
        \toprule
        \multirow{2}{*}{SGC} & \multicolumn{3}{c}{Cora} & \multicolumn{3}{c}{ Citeseer } \\
        \cmidrule(lr){2-4} \cmidrule(lr){5-7} 
        $r$ & %$0.1$ & 
        $0.3$ & $0.5$ & $0.7$ & %$0.1$ &
        $0.3$ & $0.5$ & $0.7$  \\
        \midrule
        $\Ploukas$ & 
        80.5 $\pm$ 0.0 & 79.7 $\pm$ 0.0 & 76.8 $\pm$ 0.0 & %now citeseer
        72.6 $\pm$ 0.3 & 71.7 $\pm$ 0.1 & \textbf{69.7} $\pm$ 0.7  \\
        $P_{MP}$ & 
        80.5 $\pm$ 0.0 & \textbf{80.1} $\pm$ 0.0 & 77.7 $\pm$ 0.0 &  %now citeseer
        72.8 $\pm$ 0.5 & 72.7 $\pm$ 0.0 & 69.5 $\pm$ 0.3  \\
        
        $\Prao$ &  
        77.1 $\pm$ 0.6 & 75.9 $\pm$ 0.1 & 73.8 $\pm$ 0.3 &  %now citeseer
        70.9 $\pm$ 0.2 & 70.2 $\pm$ 0.1 & 67.3 $\pm$ 0.4  \\
        
        $P_{Q^\top}^{*}$ &  
        80.3 $\pm$ 0.0 & 80.0 $\pm$ 0.1 & 77.2 $\pm$ 0.0 & %now citeseer
        72.7 $\pm$ 0.3 & 72.6 $\pm$ 0.5 & 67.6 $\pm$ 0.2  \\
         $P_{g}^{*}$&  
         \textbf{80.7} $\pm$ 0.0 & 80.0 $\pm$ 0.0 & 77.6 $\pm$ 0.0 & %now citeseer
         72.6 $\pm$ 0.2 & \textbf{72.7} $\pm$ 0.0 & 68.6 $\pm$ 0.4  \\
         
         $P_{g,l_1}^{*}$&  
         80.4 $\pm$ 0.0 & 79.2 $\pm$ 0.0 & \textbf{78.3} $\pm$ 0.0 & %now citeseer
         \textbf{73.0} $\pm$ 0.0 & 71.2 $\pm$ 0.1 & 69.2 $\pm$ 0.4 
         \\
         %Max acc possible & 96.1 & 93.1 & 88.3 & 93.5 & 90.5& 84.5\\
         Full Graph &  \multicolumn{3}{c}{81.0 $\pm$ 0.1} & \multicolumn{3}{c}{71.6 $\pm$ 0.1} \\
        %\bottomrule
    %\end{tabular}
    %\label{tab:trainingresultsSGC}
    %\vspace{-10pt}
%\end{table}
%\begin{table}[ht]
    %\vspace{-10pt}
    %\centering
    %\small
    %\caption{Accuracy in $\%$ for node classification with GCN and different coarsening ratio}
    %\smaller
    %\vspace{10pt}
    %\begin{tabular}{ccccccc}
        %\toprule
        \specialrule{1.2pt}{0pt}{2pt}
        \multirow{2}{*}{GCN} & \multicolumn{3}{c}{Cora} & \multicolumn{3}{c}{ Citeseer } \\
        \cmidrule(lr){2-4} \cmidrule(lr){5-7} 
        $r$ & %$0.1$ & 
        $0.3$ & $0.5$ & $0.7$ & %$0.1$ &
        $0.3$ & $0.5$ & $0.7$  \\
        \midrule
        $\Ploukas$ & 
        \textbf{80.6} $\pm$ 0.8 & 80.5 $\pm$ 1.0 & 78.1 $\pm$ 1.4 & %now citeseer
        71.0 $\pm$ 1.6 & 72.2 $\pm$ 0.6 & 70.4 $\pm$ 0.8  \\
        $P_{MP}$ & 
        80.4 $\pm$ 1.0 & 80.7 $\pm$ 0.9 & 78.6 $\pm$ 0.9 &  %now citeseer
        70.8 $\pm$ 1.9 & 72.1 $\pm$ 1.0 & \textbf{71.0} $\pm$ 1.0  \\
        
        $\Prao$ &  
        73.7 $\pm$ 1.5 & 63.3 $\pm$ 1.4 & 55.11 $\pm$ 2.4 &  %now citeseer
        64.6 $\pm$ 0.7 & 50.4 $\pm$ 1.6 & 42.6 $\pm$ 4.0  \\
        
        $P_{Q^\top}^{*}$ &  
        80.5 $\pm$ 0.9 & 80.9 $\pm$ 0.6 & 78.0 $\pm$ 0.9 & %now citeseer
        \textbf{71.1} $\pm$ 1.5 & \textbf{72.3} $\pm$ 0.7 & 70.0 $\pm$ 0.9  \\
         $P_{g}^{*}$&  
         \textbf{80.6} $\pm$ 1.1 & \textbf{81.3} $\pm$ 0.6 & \textbf{78.7} $\pm$ 0.9 & %now citeseer
         \textbf{71.1} $\pm$ 1.7 & 72.1 $\pm$ 1.2 & 69.6 $\pm$ 1.0  \\
         
         $P_{g,l_1}^{*}$&  
         80.4 $\pm$ 0.9 & 80.0 $\pm$ 0.9 & 78.2 $\pm$ 0.7 & %now citeseer
         70.2 $\pm$ 1.8 & 66.8 $\pm$ 1.1 & 66.7 $\pm$ 1.2 
         \\
         %Max acc possible & 96.1 & 93.1 & 88.3 & 93.5 & 90.5& 84.5\\
         Full Graph &  \multicolumn{3}{c}{81.3 $\pm$ 0.8} & \multicolumn{3}{c}{70.9 $\pm$ 1.4}  \\
        \bottomrule
    \end{tabular}
    \label{tab:trainingresultsGCN}
    \vspace{-10pt}
\end{table}

\section{Conclusion}
In this paper, we highlighted the crucial role of the lifting matrix $\Qnorm$ in graph coarsening. Surprisingly, we found that we can take advantage of the degree of freedom over the reduction matrix to obtain better spectral guarantees, without changing the coarsening itself or the lifting matrix. %We have first defined the reduction matrices for a fixed lifting matrix such the coarsen and then lift operator $\Pi = \Qnorm P$ be a projection. 
We have defined various sets of ``admissible'' reduction matrices with different properties, from the very generic property of simply obtaining a projection $\Pi$, to convenient parametrization with reflexive generalized inverses, and support constraints.
%thus characterized two subsets which are easier to optimize on: the generalized inverse of the lifting matrix $Q^g$ and those with the same support as $Q^\top$. 
We then showed that the classical choices of reduction matrices can be outperformed, both by well-motivated novel examples with analytic expressions, or by matrices resulting from various optimization processes with sparsity or support constraints.
%by RSA-optimized matrices or a dense reduction matrix $\Prao$ motivated by work on the least square problem in spaces with Mahalanobis norms.
%
Even if previous works linked the performances of GNNs trained on coarsened graph with the RSA, we empirically showed that the benefits of improving the RSA were somewhat marginal, although visible for high coarsening ratio and homophilous graphs. This suggests that other factors, such the sparsity of the reduction matrix, could also play in GNNs training.

In this work, we have selected the RSA as a general-purpose score to optimize, however our theoretical characterization of the admissible sets of reduction matrices does not particularly rely on it. We thus believe that considering more complex scoring function to take into account graphs heterophily or node features while being scalable to larger graphs is a major path for future works. Notably, our optimization framework is agnostic to the specific choice of scoring function and coarsening algorithm and directly applies to these extensions.

\acksection
The authors acknowledge the fundings of France 2030, PEPR IA, ANR-23-PEIA-0008 and  European Union ERC-2024-STG-101163069 MALAGA.
\begingroup
\small

\bibliography{refs}  

@article{hu2020open,
  title={Open graph benchmark: Datasets for machine learning on graphs},
  author={Hu, Weihua and Fey, Matthias and Zitnik, Marinka and Dong, Yuxiao and Ren, Hongyu and Liu, Bowen and Catasta, Michele and Leskovec, Jure},
  journal={Advances in neural information processing systems},
  volume={33},
  pages={22118--22133},
  year={2020}
}

@article{Bronstein2021,
archivePrefix = {arXiv},
arxivId = {2104.13478},
author = {Bronstein, Michael M. and Bruna, Joan and Cohen, Taco and Veli{\v{c}}kovi{\'{c}}, Petar},
eprint = {2104.13478},
journal = {arXiv:2104.13478},
title = {{Geometric Deep Learning: Grids, Groups, Graphs, Geodesics, and Gauges}},
url = {http://arxiv.org/abs/2104.13478},
year = {2021}
}

@article{Ediger2010,
author = {Ediger, David and Jiang, Karl and Riedy, Jason and Bader, David A. and Corley, Courtney and Farber, Rob and Reynolds, William N.},
doi = {10.1109/ICPP.2010.66},
isbn = {9780769541563},
issn = {01903918},
journal = {Proceedings of the International Conference on Parallel Processing},
pages = {583--593},
title = {{Massive social network analysis: Mining twitter for social good}},
year = {2010}
}

@article{Wang2018,
archivePrefix = {arXiv},
arxivId = {1803.02349},
author = {Wang, Jizhe and Huang, Pipei and Zhao, Huan and Zhang, Zhibo and Zhao, Binqiang and Lee, Dik Lun},
doi = {10.1145/3219819.3219869},
eprint = {1803.02349},
isbn = {9781450355520},
journal = {Proceedings of the ACM SIGKDD International Conference on Knowledge Discovery and Data Mining},
pages = {839--848},
title = {{Billion-scale commodity embedding for E-commerce recommendation in alibaba}},
year = {2018}
}

@article{loukas2019graph,
  title={Graph reduction with spectral and cut guarantees},
  author={Loukas, Andreas},
  journal={Journal of Machine Learning Research},
  volume={20},
  number={116},
  pages={1--42},
  year={2019}
}

@article{chen2022graph,
  title={Graph coarsening: from scientific computing to machine learning},
  author={Chen, Jie and Saad, Yousef and Zhang, Zechen},
  journal={SeMA Journal},
  volume={79},
  number={1},
  pages={187--223},
  year={2022},
  publisher={Springer}
}

@article{dhillon2007weighted,
  title={{Weighted graph cuts without eigenvectors a multilevel approach}},
  author={Dhillon, Inderjit S and Guan, Yuqiang and Kulis, Brian},
  journal={IEEE transactions on pattern analysis and machine intelligence},
  volume={29},
  number={11},
  pages={1944--1957},
  year={2007},
  publisher={IEEE}
}

@inproceedings{chen2023gromov,
  title={{A Gromov-Wasserstein geometric view of spectrum-preserving graph coarsening}},
  author={Chen, Yifan and Yao, Rentian and Yang, Yun and Chen, Jie},
  booktitle={International Conference on Machine Learning},
  pages={5257--5281},
  year={2023},
  organization={PMLR}
}

@article{bravo2019unifying,
  title={{A unifying framework for spectrum-preserving graph sparsification and coarsening}},
  author={Bravo Hermsdorff, Gecia and Gunderson, Lee},
  journal={Advances in Neural Information Processing Systems},
  volume={32},
  year={2019}
}

@inproceedings{jin2020graph,
  title={Graph coarsening with preserved spectral properties},
  author={Jin, Yu and Loukas, Andreas and JaJa, Joseph},
  booktitle={International Conference on Artificial Intelligence and Statistics},
  pages={4452--4462},
  year={2020},
  organization={PMLR}
}

@inproceedings{loukas2018spectrally,
  title={Spectrally approximating large graphs with smaller graphs},
  author={Loukas, Andreas and Vandergheynst, Pierre},
  booktitle={International conference on machine learning},
  pages={3237--3246},
  year={2018},
  organization={PMLR}
}

@inproceedings{kumar2023featured,
  title={{Featured graph coarsening with similarity guarantees}},
  author={Kumar, Manoj and Sharma, Anurag and Saxena, Shashwat and Kumar, Sandeep},
  booktitle={International Conference on Machine Learning},
  pages={17953--17975},
  year={2023},
  organization={PMLR}
}

@inproceedings{dickens2024graph,
  title={{Graph coarsening via convolution matching for scalable graph neural network training}},
  author={Dickens, Charles and Huang, Edward and Reganti, Aishwarya and Zhu, Jiong and Subbian, Karthik and Koutra, Danai},
  booktitle={Companion Proceedings of the ACM on Web Conference},
  pages={1502--1510},
  year={2024}
}

@inproceedings{kipf2016GCN,
  title={{Semi-Supervised Classification with Graph Convolutional Networks}},
  author={Kipf, Thomas N and Welling, Max},
  booktitle={International Conference on Learning Representations},
  year={2017}
}

@article{mccallum2000automating,
  title={Automating the construction of internet portals with machine learning},
  author={McCallum, Andrew Kachites and Nigam, Kamal and Rennie, Jason and Seymore, Kristie},
  journal={Information Retrieval},
  volume={3},
  number={2},
  pages={127--163},
  year={2000},
  publisher={Springer}
}

@article{joly2024graph,
  title={{Graph Coarsening with Message-Passing Guarantees}},
  author={Joly, Antonin and Keriven, Nicolas},
  journal={Advances in Neural Information Processing Systems},
  volume={37},
  pages={114902--114927},
  year={2024}
}

@inproceedings{cai2021graph,
  title={{Graph coarsening with neural networks}},
  author={Cai, Chen and Wang, Dingkang and Wang, Yusu},
  booktitle={9th International conference on Learning Representations},
  year={2021}
}

@inproceedings{rao1972generalized,
  title={{Generalized inverse of a matrix and its applications}},
  author={Rao, Calyampudi Radhakrishna and Mitra, Sujit Kumar},
  booktitle={Proceedings of the sixth Berkeley symposium on mathematical statistics and probability},
  volume={1},
  pages={601--620},
  year={1972},
  organization={University of California Press Oakland, CA, USA}
}

@inproceedings{huang2021scaling,
  title={{Scaling up graph neural networks via graph coarsening}},
  author={Huang, Zengfeng and Zhang, Shengzhong and Xi, Chong and Liu, Tang and Zhou, Min},
  booktitle={Proceedings of the 27th ACM SIGKDD conference on knowledge discovery \& data mining},
  pages={675--684},
  year={2021}
}

@inproceedings{penrose1955generalized,
  title={{A generalized inverse for matrices}},
  author={Penrose, Roger},
  booktitle={Mathematical proceedings of the Cambridge philosophical society},
  volume={51},
  number={3},
  pages={406--413},
  year={1955},
  organization={Cambridge University Press}
}

@incollection{ruge1987algebraic,
  title={{Algebraic multigrid}},
  author={Ruge, John W and St{\"u}ben, Klaus},
  booktitle={Multigrid methods},
  pages={73--130},
  year={1987},
  publisher={SIAM}
}

@article{Hu2013,
archivePrefix = {arXiv},
arxivId = {1308.5865},
author = {Hu, Pili and Lau, Wing Cheong},
eprint = {1308.5865},
pages = {1--34},
title = {{A Survey and Taxonomy of Graph Sampling}},
url = {http://arxiv.org/abs/1308.5865},
year = {2013}
}

@article{spielman2011spectral,
  title={{Spectral sparsification of graphs}},
  author={Spielman, Daniel A and Teng, Shang-Hua},
  journal={SIAM Journal on Computing},
  volume={40},
  number={4},
  pages={981--1025},
  year={2011},
  publisher={SIAM}
}

@inproceedings{allen2015spectral,
  title={{Spectral sparsification and regret minimization beyond matrix multiplicative updates}},
  author={Allen-Zhu, Zeyuan and Liao, Zhenyu and Orecchia, Lorenzo},
  booktitle={Proceedings of the forty-seventh annual ACM symposium on Theory of computing},
  pages={237--245},
  year={2015}
}

@article{lee2018constructing,
  title={{Constructing linear-sized spectral sparsification in almost-linear time}},
  author={Lee, Yin Tat and Sun, He},
  journal={SIAM Journal on Computing},
  volume={47},
  number={6},
  pages={2315--2336},
  year={2018},
  publisher={SIAM}
}

@inproceedings{jin2021graph,
  title={{Graph Condensation for Graph Neural Networks}},
  author={Jin, Wei and Zhao, Lingxiao and Zhang, Shichang and Liu, Yozen and Tang, Jiliang and Shah, Neil},
  booktitle={International Conference on Learning Representations},
  year={2022}
}

@inproceedings{jin2022condensing,
  title={{Condensing graphs via one-step gradient matching}},
  author={Jin, Wei and Tang, Xianfeng and Jiang, Haoming and Li, Zheng and Zhang, Danqing and Tang, Jiliang and Yin, Bing},
  booktitle={Proceedings of the 28th ACM SIGKDD Conference on Knowledge Discovery and Data Mining},
  pages={720--730},
  year={2022}
}

@article{zheng2024structure,
  title={Structure-free graph condensation: From large-scale graphs to condensed graph-free data},
  author={Zheng, Xin and Zhang, Miao and Chen, Chunyang and Nguyen, Quoc Viet Hung and Zhu, Xingquan and Pan, Shirui},
  journal={Advances in Neural Information Processing Systems},
  volume={36},
  pages={6026--6047},
  year={2023}
}

@article{wang2018dataset,
  title={{Dataset distillation}},
  author={Wang, Tongzhou and Zhu, Jun-Yan and Torralba, Antonio and Efros, Alexei A},
  journal={arXiv preprint arXiv:1811.10959},
  year={2018}
}

@article{defferrard2016convolutional,
  title={{Convolutional neural networks on graphs with fast localized spectral filtering}},
  author={Defferrard, Micha{\"e}l and Bresson, Xavier and Vandergheynst, Pierre},
  journal={Advances in neural information processing systems},
  volume={29},
  year={2016}
}

@article{dorfler2012kron,
  title={{Kron reduction of graphs with applications to electrical networks}},
  author={Dorfler, Florian and Bullo, Francesco},
  journal={IEEE Transactions on Circuits and Systems I: Regular Papers},
  volume={60},
  number={1},
  pages={150--163},
  year={2012},
  publisher={IEEE}
}

@inproceedings{ma2021unsupervised,
  title={{Unsupervised learning of graph hierarchical abstractions with differentiable coarsening and optimal transport}},
  author={Ma, Tengfei and Chen, Jie},
  booktitle={Proceedings of the AAAI conference on artificial intelligence},
  volume={35},
  number={10},
  pages={8856--8864},
  year={2021}
}

@inproceedings{giles1998citeseer,
  title={CiteSeer: An automatic citation indexing system},
  author={Giles, C Lee and Bollacker, Kurt D and Lawrence, Steve},
  booktitle={Proceedings of the third ACM conference on Digital libraries},
  pages={89--98},
  year={1998}
}

@book{ben2003generalized,
  title={{Generalized inverses: theory and applications}},
  author={Ben-Israel, Adi and Greville, Thomas NE},
  year={2003},
  publisher={Springer}
}

@inproceedings{yang2016revisiting,
  title={{Revisiting semi-supervised learning with graph embeddings}},
  author={Yang, Zhilin and Cohen, William and Salakhudinov, Ruslan},
  booktitle={International conference on machine learning},
  pages={40--48},
  year={2016},
  organization={PMLR}
}

@inproceedings{wu2019SGC,
  title={{Simplifying graph convolutional networks}},
  author={Wu, Felix and Souza, Amauri and Zhang, Tianyi and Fifty, Christopher and Yu, Tao and Weinberger, Kilian},
  booktitle={International conference on machine learning},
  pages={6861--6871},
  year={2019},
  organization={PMLR}
}

@article{hamilton2017inductive,
  title={{Inductive representation learning on large graphs}},
  author={Hamilton, Will and Ying, Zhitao and Leskovec, Jure},
  journal={Advances in neural information processing systems},
  volume={30},
  year={2017}
}

@article{ying2018hierarchical,
  title={{Hierarchical graph representation learning with differentiable pooling}},
  author={Ying, Zhitao and You, Jiaxuan and Morris, Christopher and Ren, Xiang and Hamilton, Will and Leskovec, Jure},
  journal={Advances in neural information processing systems},
  volume={31},
  year={2018}
}

@article{tsitsulin2023graph,
  title={{Graph clustering with graph neural networks}},
  author={Tsitsulin, Anton and Palowitch, John and Perozzi, Bryan and M{\"u}ller, Emmanuel},
  journal={Journal of Machine Learning Research},
  volume={24},
  number={127},
  pages={1--21},
  year={2023}
}

@inproceedings{gao2019graph,
  title={{Graph u-nets}},
  author={Gao, Hongyang and Ji, Shuiwang},
  booktitle={international conference on machine learning},
  pages={2083--2092},
  year={2019},
  organization={PMLR}
}

@article{bianchi2023expressive,
  title={The expressive power of pooling in graph neural networks},
  author={Bianchi, Filippo Maria and Lachi, Veronica},
  journal={Advances in neural information processing systems},
  volume={36},
  pages={71603--71618},
  year={2023}
}

\endgroup

%%%%%%%%%%%%%%%%%%%%%%%%%%%%%%%%%%%%%%%%%%%%%%%%%%%%%%%%%%%%

\appendix

\section{Useful property and definitions}

\begin{definition}[Left inverse]
    If the matrix \( A \) has dimensions \( m \times n \) and \(\operatorname{rank}(A) = n\), then there exists an \( n \times m \) matrix \( A_L^{-1} \), called the left inverse of \( A \), such that:
    \[
    A_L^{-1} A = I_n 
    \]
    where \( I_n \) is the \( n \times n \) identity matrix. %One example is \( A_L^{-1} = (A^*A)^{-1} A ^* \)
\end{definition}

\begin{definition}[Right inverse]
     If the matrix \( A \) has dimensions \( m \times n \) and \(\operatorname{rank}(A) = m\), then there exists an \( n \times m \) matrix \( A_R^{-1} \), called the right inverse of \( A \), such that:
    \[
    A A_R^{-1} = I_m 
    \]
    where \( I_m \) is the \( m \times m \) identity matrix. %One example is \( A_R^{-1} =  A^* (A^*A)^{-1} \)
\end{definition}
%\begin{remark}
%Please note that left and right inverse are not unique but that Moore Penrose inverse is one of them if $rank(A)=n$.
%\end{remark} not necessarly
\begin{lemma}[rank of well-partitioned matrices]
\label{lem:wellmap}
For a well-partitioned lifting matrix $Q \in \mathbb{R}^{N \times n}$, 
\begin{equation}
    \operatorname{rank}(Q) = n
\end{equation}
\end{lemma}
\begin{proof}
For a well-partitioned lifting matrix $Q \in \mathbb{R}^{N \times n}$, there is only one non zero value per row. Consequently all the columns are independent and the rank of the rank of this matrix is equal to $n$.
\end{proof}
\begin{remark}
  Similarly a reduction matrix   $P \in \mathbb{R}^{n \times N}$, which has the same support as $Q^\top$ with $Q$ a well-partitioned lifting matrix is also of rank $n$.
\end{remark}

\begin{lemma}[$QQ^\top $ is diagonal]
\label{lem::qqtdiag}
Let $Q$ be a well-partitioned lifting matrix. For all matrix $P\in \mathbb{R}^{n\times N}$ with the same support as $Q^\top$, $PQ$ is diagonal.
\end{lemma}
\begin{proof}
Let $P$ a reduction matrix with the same support as $Q^\top$.
 Let's show that ($Q^\top Q$ is diagonal) : 

For $i,j \in \{1\ldots N\}^2$ and $i \neq j $:
\begin{align*}
    (Q^\top Q)_{ij} &= \sum_{k=1}^N Q_{kj}Q_{ki}\\
    &= 0
\end{align*}
Indeed if one term is different from $0$, it means that two nodes $i$ and $j$ of $G_c$ have been expanded to a same node $k$ of $G$ which contradicts the well-partitioned definition.

Thus $Q^\top Q$ is diagonal and by equality of the support $PQ$ is also diagonal.
\end{proof}
\begin{lemma}[$PQ = I_n $]
\label{lem::pqin}
Let $Q$ be a well-partitioned lifting matrix. We have the following equivalence :
\[
    PQ = I_n \Longleftrightarrow P \in Q^g 
\]
\end{lemma}
\begin{proof}
Let's show the two sides of the equivalence.
\paragraph{$\Longrightarrow$} If $PQ = I_n$, we have directly $\Pi^2 = \Pi$ and thus with Lem.~\ref{lem:piproj},  $Q \in P^g$. Then as $PQ = I_n$, we have $\operatorname{rank}(P) = n $ and using Lem.~\ref{lem::genrefinv} as  $Q \in P^g$ we have $P \in Q^g$
\paragraph{$\Longleftarrow$} if $P \in Q^g$ then $QPQ = Q$, by multiplying by the Moore Penrose inverse $Q^+$ to the left, we have $Q^+QPQ = Q^+Q$ thus $PQ = I_n$. The Moore Penrose inverse being a left inverse of $Q$ as seen in App.~\ref{app:proofmp}
\end{proof}

\section{Proofs} \label{app:proof}

\subsection{Consistency Lem.~\ref{lem:consistency}}
\label{app:proofconsist}
This property is an extension of Loukas work \cite{loukas2019graph} on the consistency of the combinatorial Laplacian. We will use the Lem.~\ref{lem:loukas} proven in \cite{loukas2019graph} for our proof. 

\begin{comment}
\begin{lemma}[Textual description Loukas consistency]
Let $Q$ be a lifting matrix w.r.t. a graph with combinatorial Laplacian $\Lcombi$. Matrix $\Lcombi_c = Q^\top \Lcombi Q$ is a combinatorial Laplacian if and only if the
non-zero entries of $Q$ are equally valued.
\end{lemma}

\begin{lemma}[\cite{loukas2019graph}]
Let $\Qcombi$ be a well-partitioned lifting matrix. The two following properties are equivalent: a) $\Qcombi$ is proportional to a binary matrix; b) for all adjacency matrix $A$, $\Lcombi_c := \Lcombi(A_c) = \Qcombi^\top \Lcombi(A) \Qcombi$.
\end{lemma}
\end{comment}

Let $\Qcombi$ be a well-partitioned lifting matrix. For all adjacency matrix $A$, consider $\Lnorm(A) = \diagmatrix(A)\Lcombi(A)\diagmatrix(A)$ and $\Qnorm(A,\Qcombi) = {\diagmatrix(A)}^{-1} \Qcombi \diagmatrix(A_c) = \diagmatrix^{-1} \Qcombi \diagmatrix_c$.

\paragraph{$\Longleftarrow $}
    When $\Qcombi$ is a binary matrix. By using Lem.~\ref{lem:loukas} $\Lcombi(A_c) = \Qcombi^\top \Lcombi(A) \Qcombi$. Thus, 
    
    \begin{align*}
        \Qnorm^\top  \Lnorm \Qnorm &= {\Qnorm(A,\Qcombi)}^\top  \Lnorm(A) \Qnorm(A,\Qcombi)\\
        &=  \diagmatrix_c \Qcombi^\top  \diagmatrix^{-1}   \diagmatrix \Lcombi(A) \diagmatrix^{-1} \Qcombi \diagmatrix_c \\
        &= \diagmatrix(A_c) \Qcombi^\top   \Lcombi(A)  \Qcombi \diagmatrix(A_c) \\
        &= \diagmatrix(A_c) \Lcombi(A_c) \diagmatrix(A_c) \\
        &= \Lnorm(A_c) \\
        &= \Lnorm_c 
    \end{align*}

\paragraph{$\Longrightarrow$ }
When $\Lnorm_c = \Qnorm^\top   \Lnorm \Qnorm$ then
\begin{align*}
    \Lnorm_c =  \Qnorm^\top   \Lnorm \Qnorm &\Longrightarrow \diagmatrix_c \Lcombi_c \diagmatrix_c =  \diagmatrix_c \Qcombi^\top  \diagmatrix^{-1}  \diagmatrix \Lcombi \diagmatrix \diagmatrix^{-1} \Qcombi \diagmatrix_c \\
    &\Longrightarrow  \diagmatrix_c \Lcombi_c \diagmatrix_c =   \diagmatrix_c \Qcombi^\top   \Lcombi \Qcombi  \diagmatrix_c \\
    &\Longrightarrow \Lcombi_c = \Qcombi^\top   \Lcombi \Qcombi \\
    &\Longrightarrow \Qcombi \, \text{is binary}
\end{align*}

The last line using Lem.~\ref{lem:loukas}.

\subsection{Proof of Lemma Moore Penrose inverse}
\label{app:proofmp}

\begin{proof}
Let's show that $Q^+ = (Q^\top Q)^{-1}Q^\top  $. 

For a given well-partitioned lifting matrix $Q \in \mathbb{R}^{N\times n}$ , we have $\operatorname{rank}(Q) = n$ as proposed in Lem.~\ref{lem:wellmap}.

We can compute one left inverse as $Q_L^{-1} = (Q^\top Q)^{-1}Q^\top $.

It verifies the four properties of Moore Penrose inverse : 
\begin{enumerate}
    \item $QQ_L^{-1}Q = Q(Q^\top Q)^{-1}Q^\top Q =Q  $
    \item $Q_L^{-1}QQ_L^{-1} = (Q^\top Q)^{-1}Q^\top  Q (Q^\top Q)^{-1}Q^\top  = (Q^\top Q)^{-1}Q^\top  = Q_L^{-1}$
    \item $(QQ_L^{-1})^\top  = ( (Q^\top Q)^{-1}Q^\top )^\top  Q^\top  = Q ((Q^\top Q)^{-1})^\top  Q^\top  = QQ_L^{-1}$ 
    \item $(Q_L^{-1}Q)^\top  = I_n = Q_L^{-1}Q$ 
\end{enumerate}
Thus it is the unique Moore Penrose inverse of the lifting matrix $Q$.

Please note that for the third condition we used that $(Q^\top Q)$ diagonal (see Lem.~\ref{lem::qqtdiag}) and so $(Q^\top Q)^{-1}$ is symmetric.

It is thus by definition a generalized reflexive inverse of $Q$ with the same support as $\Qnorm^\top$.
\end{proof}

\subsection{Proof $\Ploukas$ generalized reflexive inverse }
\label{app:ploukas}
For a multilevel coarsening scheme we have intermediary coarsening with intermediary well-partitioned $Q_i$. We remark that $Q = Q_1 \ldots Q_c$ is also a well-partitioned lifting matrix.

Let's examine $\Ploukas = Q_c^+ \ldots Q_1^+$: First it is of same support as $Q^\top $. Secondly, $\Ploukas Q =  Q_c^+ \ldots Q_1^+  Q_1 \ldots Q_c = I_n  $ as Moore Penrose inverse are left inverse. Using Lem.~\ref{lem::pqin}, we have $P \in Q^g$. Thus $\Ploukas$ is a generalized reflexive inverse of $Q$ and has the same support as $Q^\top $.

\subsection{Proof of Coarsen-lift operator projection Lem.~\ref{lem:piproj}}
\label{app:proofpiproj}

\begin{proof}
Let's show the two side of the equivalence :

\paragraph{ $\Longrightarrow$ $\Pi$ projection implies  $Q \in P^g$ }:

We have $\Pi^2 =  \Pi$ then $QPQP = QP$.

As $Q$ is well-partitioned, we have that $\operatorname{rank}(Q)=n$ (using Lem.~\ref{lem:wellmap}). Thus we have the existence of a left inverse ( one example is the Moore Penrose inverse) such that $Q_L^{-1}Q = I_n$. 

We multiply our expression at the left by $Q_L^{-1}$, then $Q_L^{-1}QPQP = Q_L^{-1}QP$ thus $PQP = P$. Consequently $Q$ is a generalized inverse of $P$.

\paragraph{ $\Longleftarrow$ $Q \in P^g$ implies  $\Pi = QP$ projection}:

$\Pi \Pi = QPQP = QP = \Pi$. Thus $\Pi$ is a projection using directly the generalized inverse property on $Q$. (condition (i) ).
\end{proof}
\subsection{Proof of Reflexive generalized inverse Lem.~\ref{lem::genrefinv}}
\label{app:prfreflexinv}
%First we recall the theorem :

%\begin{lemma}[Reflexive generalized inverse]
%For a well-partitioned lifting matrix $Q$ and a reduction matrix $P$ such that $Q \in P^g$, we have the following equivalence :
%\[
%\operatorname{rank}(P) = n \quad \Longleftrightarrow
% \quad P \in Q^g
%\]
%\end{lemma}

\begin{proof}

For a well-partitioned lifting matrix $Q$ and a reduction matrix $P$ such that $Q \in P^g$. Let's show the two sides of the equivalence:

\paragraph{ $\Longrightarrow$ $\operatorname{rank}(P) = n$ implies  $P \in Q^g$ }:

With this "full rank" reduction matrix, using Lem.~\ref{lem:wellmap} there is the existence of a right pseudo inverse $P_R^{-1}$ such that $PP_R^{-1}= I_n$.

As we have $Q \in P^g$ that implies $\Pi^2 = \Pi$, using the Lem.~\ref{lem:piproj}. We thus have $QPQP = QP$ using the existence of $P_R^{-1}$ we multiply this equality by $P_R^{-1}$.
\begin{align*}
     QPQP = QP &\Longrightarrow QPQPP_R^{-1} = QPP_R^{-1}\\
    &\Longrightarrow QPQ = Q
\end{align*}
Thus $P\in Q^g$.

\paragraph{ $\Longleftarrow$ $P \in Q^g$ implies $\operatorname{rank}(P) = n$   }:

$QPQ = Q$ by assumption. As we know that $\operatorname{rank}(Q)=n$, we can compute one left inverse $Q_L^{-1}$ and multiply at left for both side of this equation. Thus we have $PQ = I_n$.

This is only possible if $\operatorname{rank}(P) \geq n$ otherwise the kernel of $PQ$ would not be null. Thus bounded by the dimension of the matrix $\operatorname{rank}(P) = n$.

\paragraph{Second implication of the theorem}
Now for the second part of the theorem: we have $P \in Q^g$. 
Thus $PQPQ = PQ$ and $\Pi^2 = \Pi$. 

Using the Lem.~\ref{lem:piproj}, we have $Q\in P^g$. this justifies the inclusion of $E_2 \subset E_1$
\end{proof}

\subsection{Proof of Generalized inverse Characterization Lem.~\ref{lem:caracginv}}
\label{app:caracginv}
%First recall the Lemma :

%\begin{lemma}[ Generalized reflexive inverse Characterization]
    %Let $Q \in \mathbb{R}^{N \times n}$ a well-partitioned lifting matrix. All its reflexive generalized inverse can be characterized as
    %\[
    %Q^g =\{ Q^+  + M\,(I_N - QQ^+) ~|~ M \in \mathbb{R}^{n \times N}\}
    %\]
    %with $Q^+$ the Moore Penrose inverse of $Q$.
%\end{lemma}

\begin{proof}

We will use the following theorem presented as theorem 2.1 in %Rao work
\cite{rao1972generalized}, but presented as a Corollary of %Penrose
\cite{penrose1955generalized} in %Ben-Israel and Greville book
\cite{ben2003generalized}.
\begin{theorem}[Generalized inverse characterization]
Let $A \in \mathbb{R}^{m\times n}$. Then $A^{g}$ exists. The entire classe of generalized inverses is generated from any given inverse $A^{g}$ by the formula
\begin{equation}
    \label{eq:generategeneralized}
    A^{g} + U - A^{g}AUAA^{g}
\end{equation}
where $U \in \mathbb{R}^{n\times m} $ is arbitrary.
\end{theorem}

We apply this characterization to the well-partitioned lifting matrix $\Qnorm \in \mathbb{R}^{N \times n}$ using a well known generalized inverse of $\Qnorm$, namely the "Moore-Penrose pseudo inverse" $Q^+$ that has been characterized in App.~\ref{app:proofmp}.

Thus $Q^+ + M - Q^+QMQQ^+$, for $M \in \mathbb{R}^{n\times N}$ arbitrary, generates all the generalized inverse of the lifting matrix $Q$.

But this formula has some simplifications. Indeed $Q^+Q = I_n$ as we have proved in App.~\ref{app:proofmp}.  Thus the characterization can be rewritten as $Q^+ +  M (I_n - QQ^+)$ for an arbitrary $M \in \mathbb{R}^{n\times N}$.

\paragraph{Reflexive} Let's show that these generalized inverse are also reflexive : 

We have $QPQ = Q$ as $P \in Q^g$. Thus $\Pi^2 = QPQP = QP = \Pi$. Using the equivalence of Lem.~\ref{lem:piproj} we have that $Q \in P^g$. This is a characterization of reflexive generalized inverse.

\end{proof}

\subsection{Proof of Generalized reflexive inverse of same support Lem.~\ref{lem:caracsupp}}
\label{app:caracsupp}

\begin{proof}
Let's show the two side of the equivalence :

\paragraph{ $\Longrightarrow$ $P \in Q^g$ with same support implies $\sum_{k = 1}^N \frac{P_{ik}}{\diagmatrix(k)} = \frac{1}{ \diagmatrix_c(i) }$   }:

For a well-partitioned, degree-wised valued matrix $Q$ and a matrix $P \in Q^g$ with same support as $Q^\top $, using Lem.~\ref{lem::pqin}, we have $PQ = I_N$.
Thus each diagonal term must be equal to one :

\begin{align*}
    \forall i, \, (PQ)_{ii} = 1  
        &\Longrightarrow \forall i, \, \sum_{k = 1}^N P_{ik}Q_{ki} = 1 \\
        &\Longrightarrow \forall i, \, \sum_{k = 1}^N P_{ik}\frac{\diagmatrix_c(i)}{\diagmatrix(k)} = 1 \\
        &\Longrightarrow \forall i, \, \sum_{k = 1}^N \frac{P_{ik}}{\diagmatrix(k)}  = \frac{1}{ \diagmatrix_c(i) } \\
\end{align*}

It is the normalization condition we have. 

\paragraph{ $\Longleftarrow$  $\sum_{k = 1}^N \frac{P_{ik}}{\diagmatrix(k)} = \frac{1}{ \diagmatrix_c(i) }$  and with same support as $Q^\top $ implies $P \in Q^g$}:

For a reduction matrix $P$ with the same support as $Q^\top $, $PQ$ is thus a diagonal matrix (Lem.~\ref{lem::qqtdiag}).

Moreover  $\sum_{k = 1}^N \frac{P_{ik}}{\diagmatrix(k)} = \frac{1}{ \diagmatrix_c(i) }$ we have
\begin{align*}
    (PQ)_{ij} &= \sum_{k=1}^N P_{ik} Q_{kj} \\
    &= \sum_{k | Q_{kj} \neq 0, Q_{ki} \neq 0} P_{ik} \frac{\diagmatrix_c(i)}{\diagmatrix(k)}\\
    &= \begin{cases}
     0 \text{ when } i \neq j\\
    1\text{ when} i = j 
    \end{cases}
\end{align*}

Thus $PQ = I_n$, using Lem.~\ref{lem::pqin} we have $P \in Q^g$.

\end{proof}

\section{Presentation of datasets}\label{app::datasets}

We restrict the well known Cora and Citeseer to their principal connected component(PCC) as it more compatible with coarsening as a preprocessing step. Indeed, the loukas algorithm tend to coarsen first the smallest connected components before going to the biggest which leads to poor results for coarsening with a small coarsening ratio. However working with this reduced graph make the comparison with other works more difficult as it is not the same training and evaluating dataset. The characteristics of these datasets and their principal connected component are reported in Tab.~\ref{tab:caracCora}. 

\begin{table}[ht]
\centering
\caption{Characteristics of Cora and CiteSeer Datasets}
\vspace{10pt}
\begin{tabular}{lccccc}
\toprule
\textbf{Dataset} & \textbf{\# Nodes} & \textbf{\# Edges}  & \textbf{\# Train Nodes} & \textbf{\# Val Nodes} & \textbf{\# Test Nodes} \\
\midrule
Cora & 2,708 & 10,556 & 140 & 500 & 1,000   \\
Cora PCC & 2,485 & 10,138  & 122 & 459 & 915 \\
\midrule
Citeseer & 3,327 & 9,104 & 120 & 500 & 1,000   \\
Citeseer PCC & 2,120 & 7,358   & 80 & 328 & 663 \\

\bottomrule
\label{tab:caracCora}
\end{tabular}
\end{table}

\section{Random Geometric graphs}
\label{app:random_geometric}

A random geometric graph is built by sampling nodes with coordinates in $[0,1]^2$ and connecting them if their distance is under a given threshold. For this additional experiment on minimizing the RSA, we sample 1000 nodes with a threshold of $0.05$ (Fig.~\ref{fig:syn-graph}).
\begin{figure}[h!]
    \centering
    \includegraphics{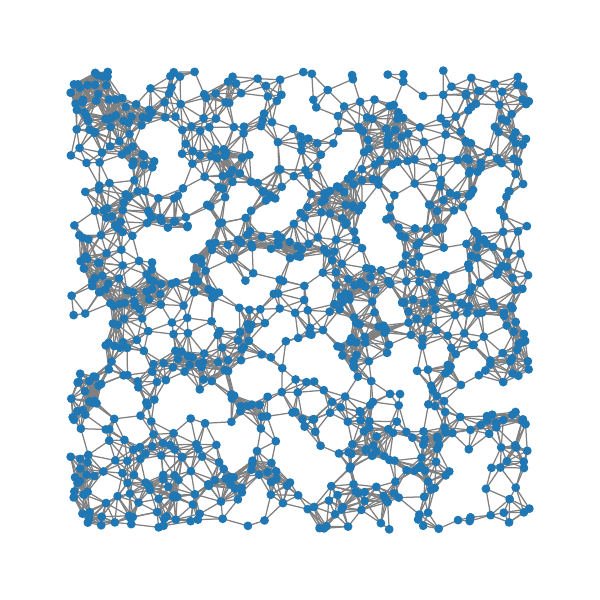}
    \caption{Example of a random Geometric graph }
    \label{fig:syn-graph}
\end{figure}

The results presented in Fig.~\ref{fig:random_geom_rsa} confirms the observation made for the same experiment on Cora, $\Prao$ being the best option to minimize the RSA.
\begin{figure}[H]
    \centering
    \begin{subfigure}[b]{0.45\textwidth}
        \includegraphics[width=\linewidth]{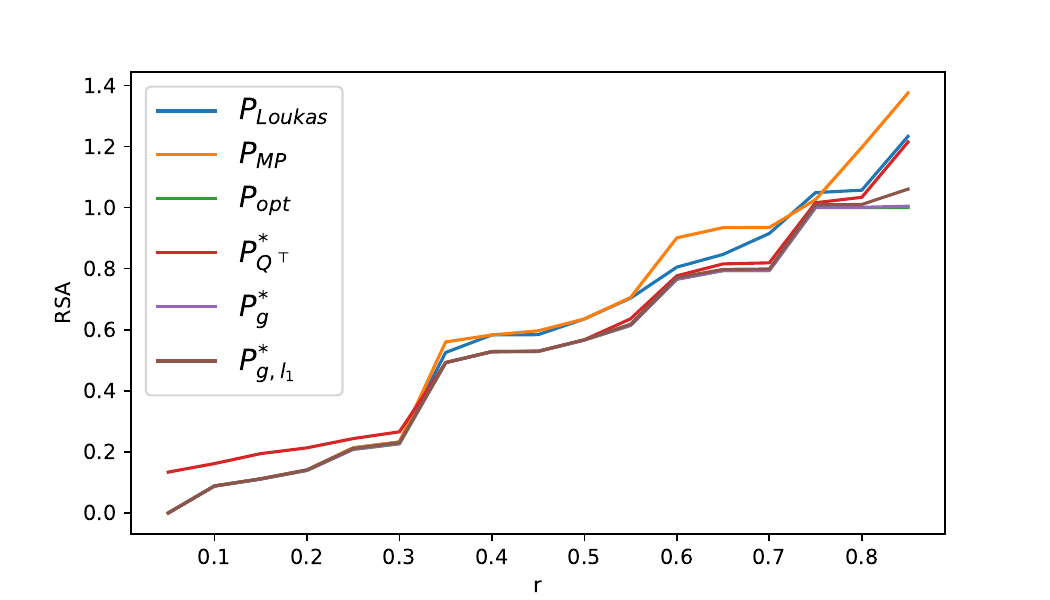}
        \caption{$\Lnorm$}
    \end{subfigure}
    \hfill
    \begin{subfigure}[b]{0.45\textwidth}
        \includegraphics[width=\linewidth]{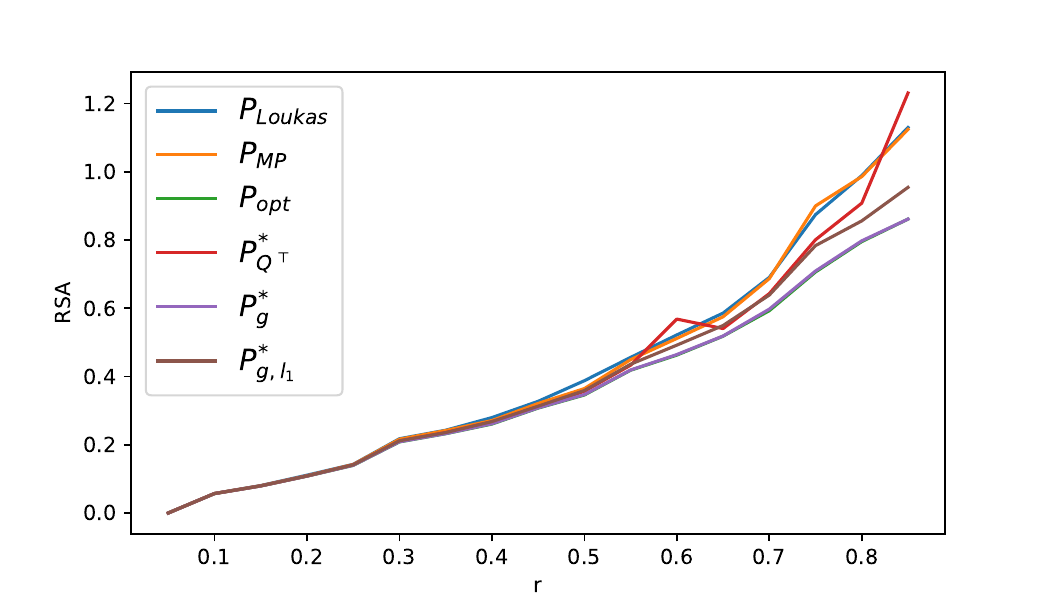}
        \caption{$\Lcombi$}
    \end{subfigure}
    \caption{RSA minimization on a Random Geometric Graph}
    \label{fig:random_geom_rsa}
\end{figure}

\section{Convergence of optimization algorithm}
\label{app:optimparam}
For the optimization hyperparameters, we compared different optimizer, different learning rate, different initialization. The results for the three different problems can be find below.
\subsection{Parameters for $P^*_{g}$}
We present here the chosen parameters for the optimization problem defined in Eq.~\ref{eq:optimRSAginv}. 

As initialization matrix $M$ for our optimization procedure, we choose the Moore Penrose inverse matrix $P_{MP}$  compared to $\Prao$ and $\Ploukas$. Indeed as shown in Fig.~\ref{fig:p_optim_g}, the $\Prao$ initialization is too close to the minima and thus the matrix doesn't change and is too similar to $\Prao$. $P_{MP}$ is a better initialization for combinatorial Laplacian as $\Ploukas$. Furthermore $P_{MP}$ is more generic as $\Ploukas$ which can only be computed in a multi-level coarsening algorithm. The random initialization does not converge fast enough and show the relevance to consider a more "classic" reduction matrix as initialization. 

We choose the stochastic gradient optimizer (SGD) as it is more stable than the Adam optimizer. For the Learning rate we choose $lr=0.01$ for an improved stability.

Fig.~\ref{fig:p_optim_g} is computed for a coarsening on Cora with $r = 0.5$.

\begin{figure}[H]
    \centering
    \begin{minipage}[b]{0.24\textwidth}
        
        \centering
        \includegraphics[width=\textwidth]{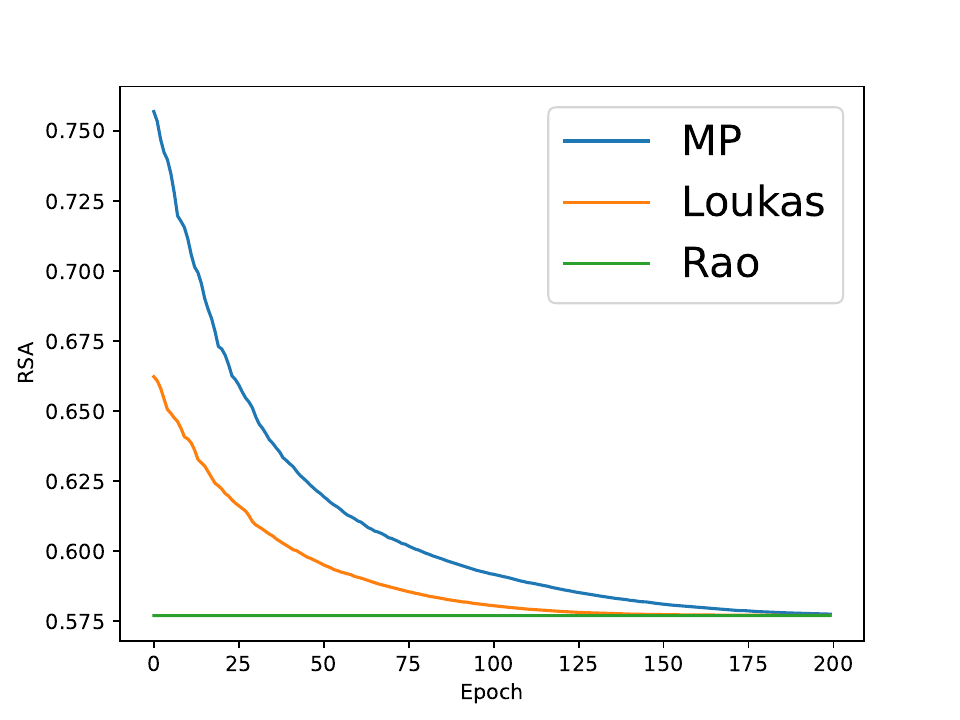}
        \caption*{(a) Initialization}
    \end{minipage}
    \hfill
    \begin{minipage}[b]{0.24\textwidth}
        
        \centering
        \includegraphics[width=\textwidth]{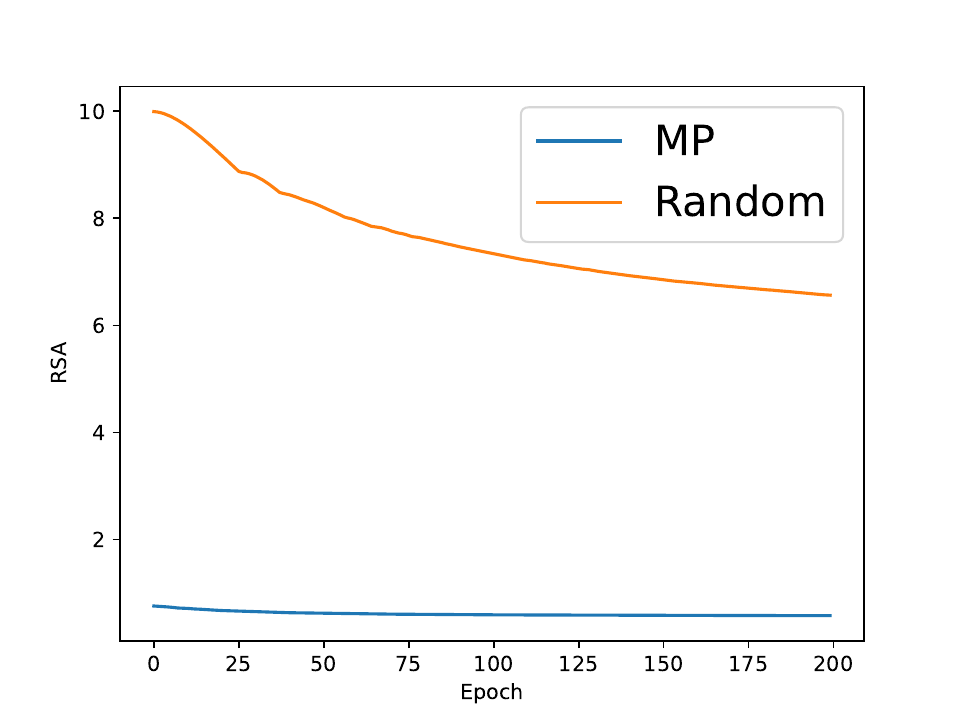}
        \caption*{(b) Random Init}
    \end{minipage}
    \hfill
    \begin{minipage}[b]{0.24\textwidth}
        \centering
        \includegraphics[width=\textwidth]{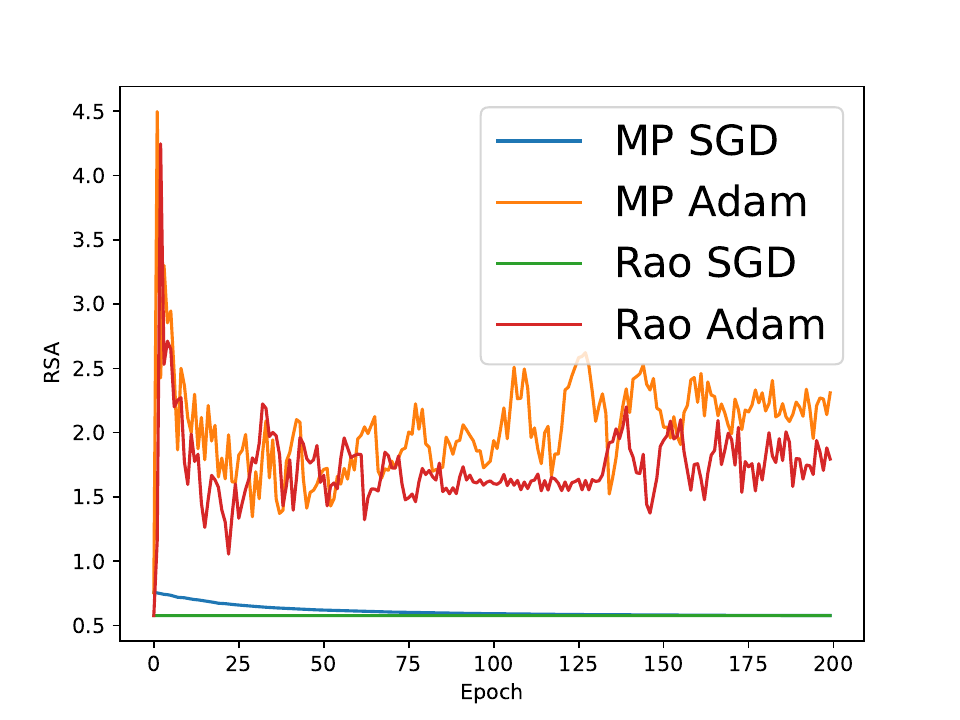}
        \caption*{(c) Optimizer}
    \end{minipage}
    \hfill
    \begin{minipage}[b]{0.24\textwidth}
        \centering
        \includegraphics[width=\textwidth]{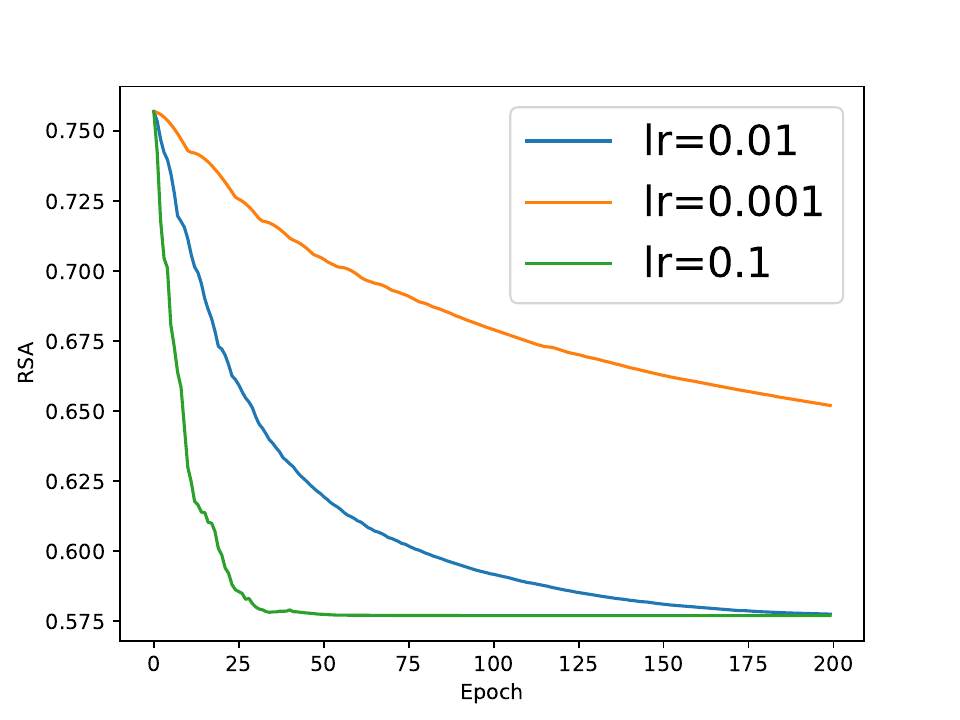}
        \caption*{(d) Learning rate}
    \end{minipage}
    \caption{Parameters for $P^*_g$}
    \label{fig:p_optim_g}
\end{figure}

\subsection{Parameters for $P^*_{g, l_1}$}

We present here the chosen parameters for the optimization problem defined in Eq.~\ref{eq:optimRSAginvsparse}.

To enforce the sparsity we apply after the optimization procedure a threshold of $0.001$ to erase the small coefficients.

We choose as initialization matrix $M$ the $\Prao$ matrix. $P_{MP}$ and $\Ploukas$ provide a better combined loss as shown in Fig.~\ref{fig:p_optim_g_sparse}, but due to the sparsity constraint they do not leave their support and the sparsity remains unchanged (as shown in Tab.~\ref{tab:sparse}). It then becomes an optimization on the same support and the RSA is less improved. Oppositely when we initialize with $P_rao$ which is very dense, we obtain a number of non zero coefficients close to the support of $P_{MP}$ with an additional $500$ coefficients. 

We choose an $\lambda$ coefficients which control the $l_1$ penalty equals to $\lambda = 0.01$ to enforce the sparsity. 

As the previous experiment, we choose the SGD optimizer and a learning rate $lr=0.01$ for an improved stability.

Fig.~\ref{fig:p_optim_g_sparse} and Tab.~\ref{tab:sparse} are computed for a coarsening on Cora with $r = 0.5$.

\begin{figure}[H]
    \centering
    \begin{minipage}[b]{0.24\textwidth}
        
        \centering
        \includegraphics[width=\textwidth]{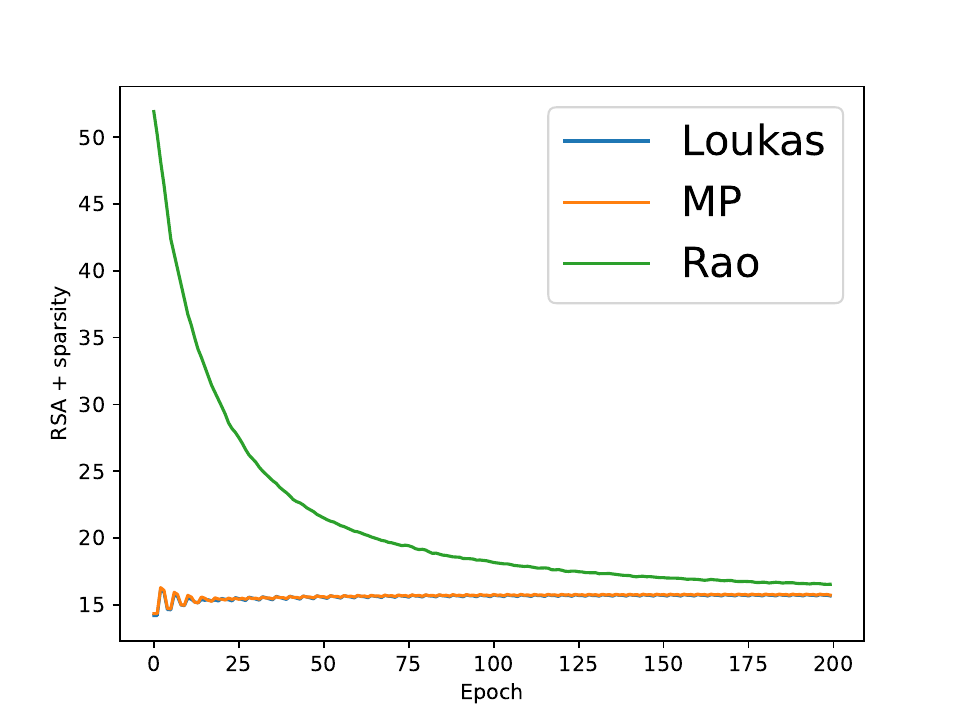}
        \caption*{(a) Initialization}
    \end{minipage}
    \hfill
    \begin{minipage}[b]{0.24\textwidth}
        
        \centering
        \includegraphics[width=\textwidth]{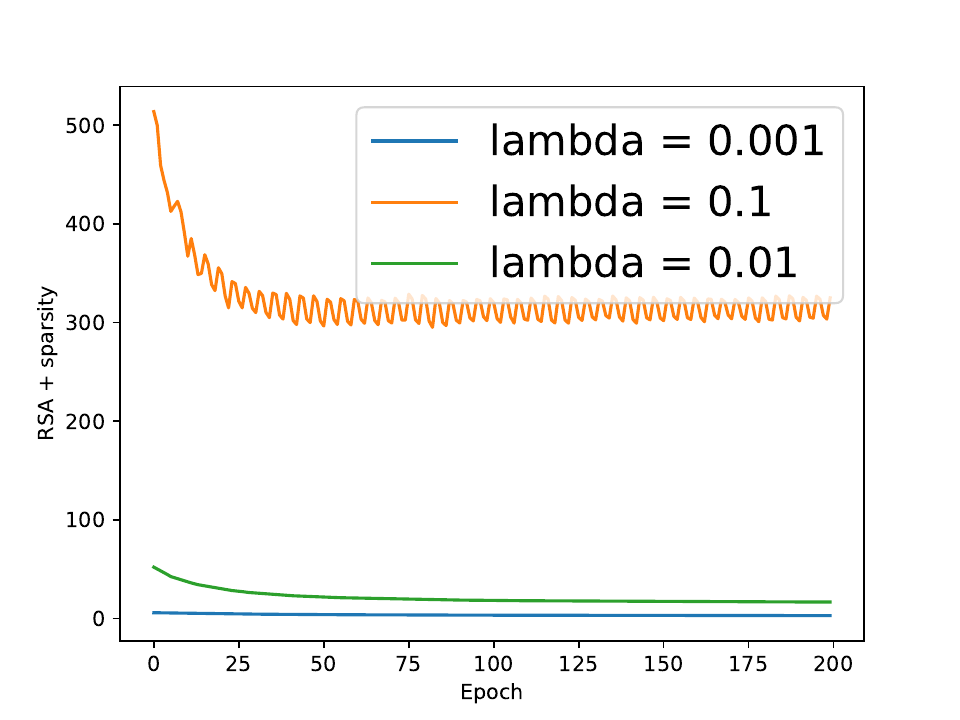}
        \caption*{(b) $\lambda$ Sparsity}
    \end{minipage}
    \hfill
    \begin{minipage}[b]{0.24\textwidth}
        \centering
        \includegraphics[width=\textwidth]{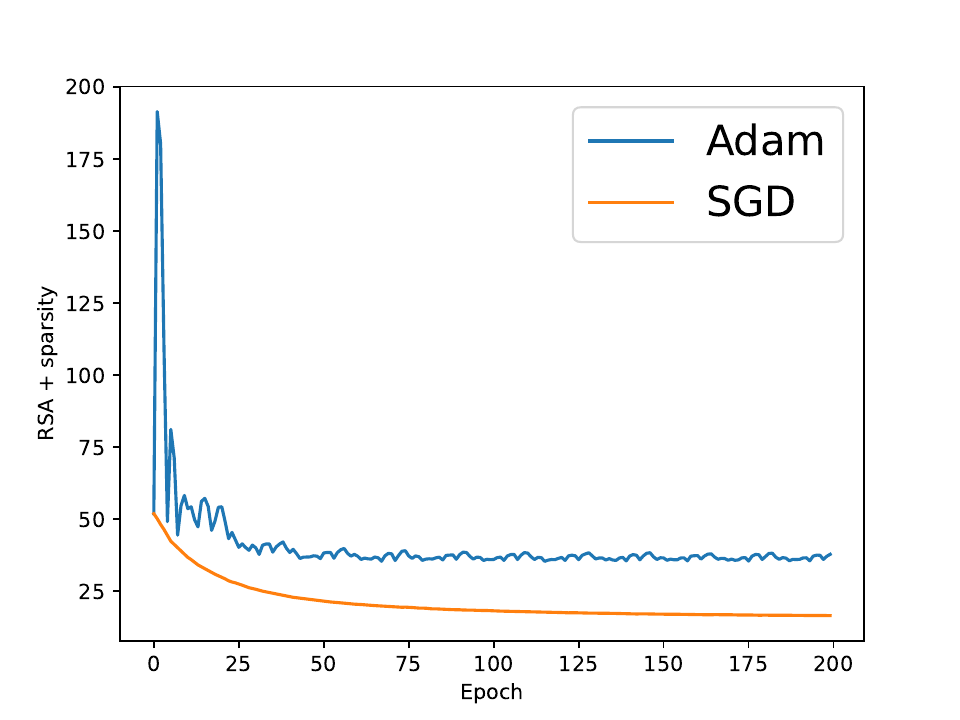}
        \caption*{(c) Optimizer}
    \end{minipage}
    \hfill
    \begin{minipage}[b]{0.24\textwidth}
        \centering
        \includegraphics[width=\textwidth]{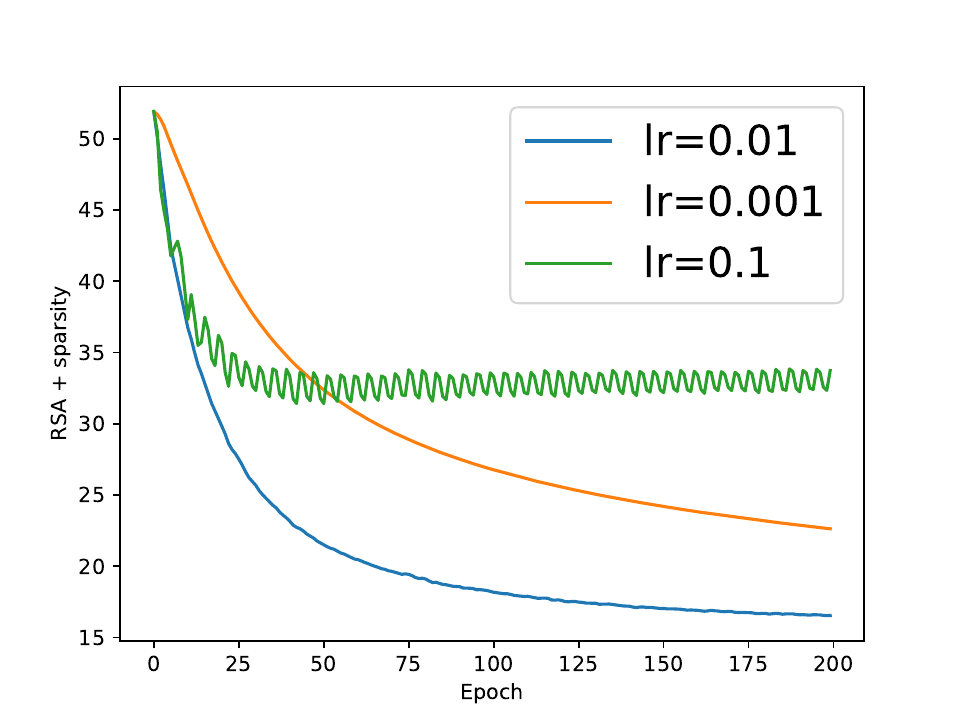}
        \caption*{(d) Learning rate}
    \end{minipage}
    \caption{Parameters for $P^*_g$}
    \label{fig:p_optim_g_sparse}
\end{figure}

\begin{table}[H]
\centering
\caption{Influence of Sparsity coefficient $\lambda$ for $P^*_{g,l_1}$ (ex with $r=0.5$ and threshold $= 0.001$) }
\vspace{10pt}
\begin{tabular}{cccc}
\toprule
 \textbf{Initialization} & \textbf{$\lambda$} &\textbf{\#Non zero coefficients of Initialization} &\textbf{\# Non zero Coefficients}  \\
\midrule
$\Ploukas$ & 0.01 & 2,485 &  2,479\\
$P_{MP}$ & 0.01 & 2,485  &  2,509 \\
$\Prao$ & 0.01 & 3,080,144 & 3,007   \\
$\Prao$ & 0.001 & 3,080,144 & 29,585   \\
$\Prao$ & 0.1 & 3,080,144 & 654,080   \\
\bottomrule
\end{tabular}
\label{tab:sparse}
\end{table}

\subsection{Parameters for $P^*_{Q^\top}$}
We present here the chosen parameters for the optimization problem defined in Eq.~\ref{eq:optimRSAsupp}. 

As initialization vector $\mu$ for our optimization procedure, we choose the non zero coefficients of the Moore Penrose inverse matrix $P_{MP}$ compared to the non zero coefficient of $\Ploukas$, a random initialization and the uniform vector which has all same values for each node in the same super node. Indeed as shown in Fig.~\ref{fig:p_optim_sup}, the $P_{MP}$ vector initialization is the only stable method.

The SGD optimizer is still more stable than the Adam optimizer which motivate its choice. For the Learning rate we choose $lr=0.05$ for a fast convergence and an improved stability. The results are reported in Fig.~\ref{fig:p_optim_sup}

\begin{figure}[H]
    \centering
    \begin{minipage}[b]{0.32\textwidth}
        
        \centering
        \includegraphics[width=\textwidth]{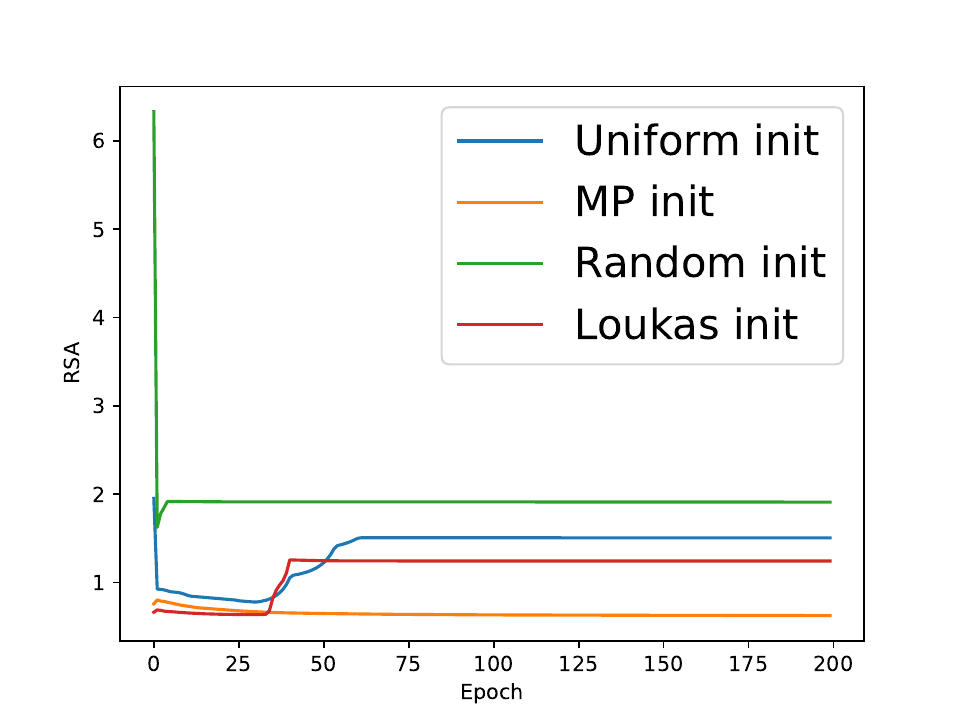}
        \caption*{(a) Initialization}
    \end{minipage}
    \hfill
    \begin{minipage}[b]{0.32\textwidth}
        \centering
        \includegraphics[width=\textwidth]{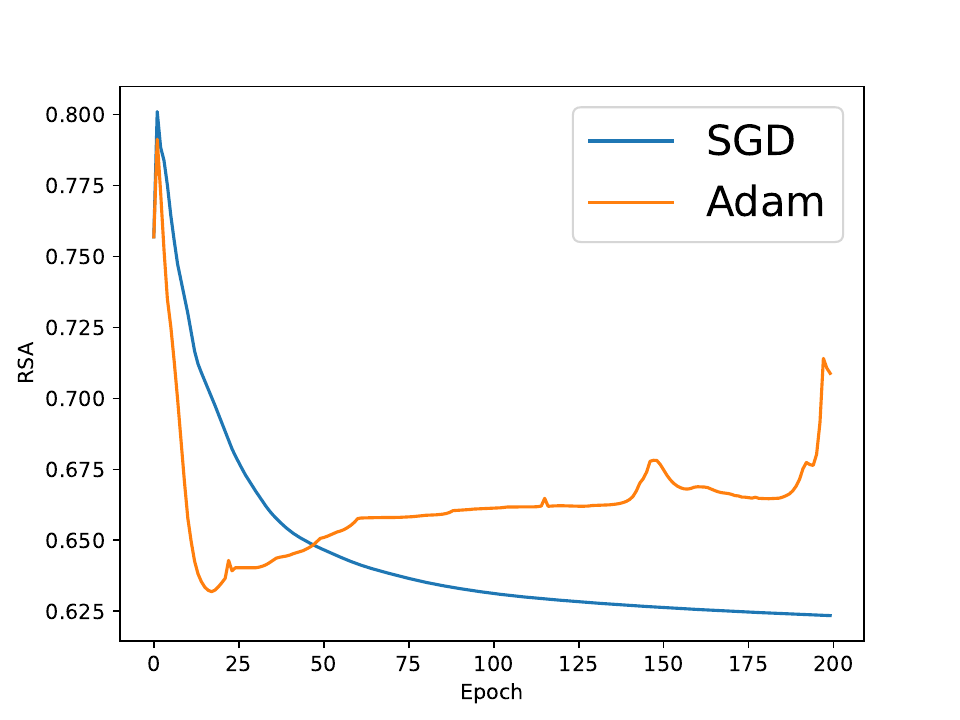}
        \caption*{(b) Optimizer}
    \end{minipage}
    \hfill
    \begin{minipage}[b]{0.32\textwidth}
        \centering
        \includegraphics[width=\textwidth]{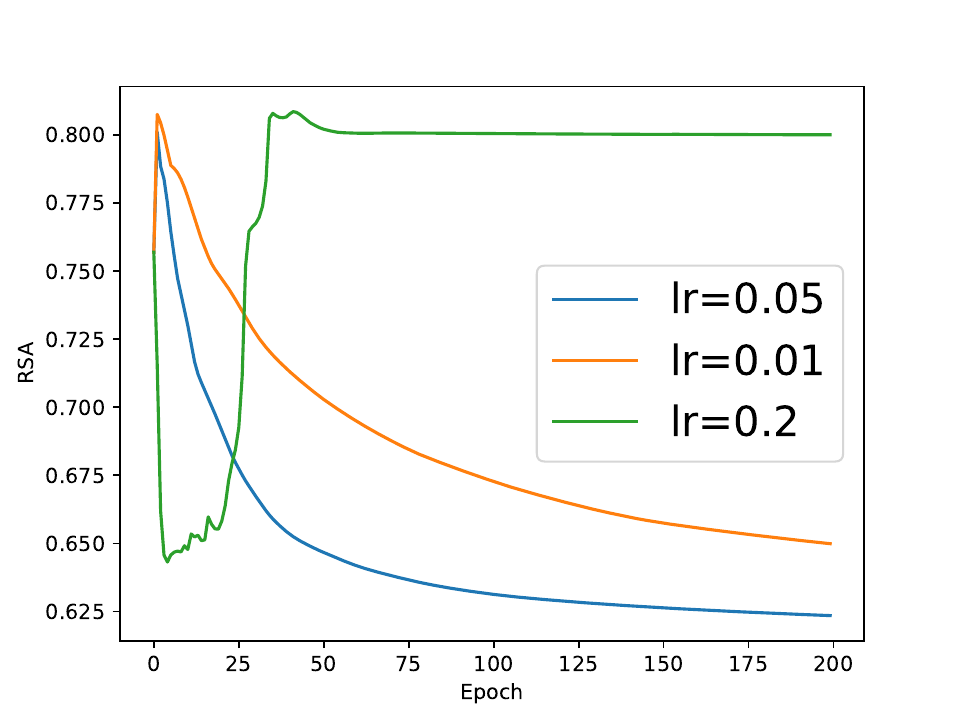}
        \caption*{(c) Learning rate}
    \end{minipage}
    \caption{Parameters for $P^*_{Q^\top}$}
    \label{fig:p_optim_sup}
\end{figure}
Fig.~\ref{fig:p_optim_sup} is computed for a coarsening on Cora with $r = 0.5$.

\section{Adaptation of Loukas Algorithm}\label{app:loukasalg}

You can find below the pseudo-code of Loukas algorithm Alg.~\ref{algo:loukas_adapted}. This algorithm works by greedy selection of \emph{contraction sets} (see below) according to some cost, merging the corresponding nodes, and iterate. The main modification is to replace the combinatorial Laplacian in the Loukas code by any Laplacian $L = \diagmatrix\Lcombi \diagmatrix$. Note that we also remove the diagonal of $A_c$ at each iteration, as Loukas given its lower value of RSA. The output of the algorithm is the resulting lifting matrix $\Qnorm = \Qnorm_1 \ldots \Qnorm_c$, the coarsened adjacency $A_c$ and the the Loukas and Moore Penrose reduction matrices $\Ploukas$ and $P_{MP}$.

\begin{algorithm}
\caption{Loukas Algorithm}
\label{algo:loukas_adapted}
\begin{algorithmic}[1]
    \REQUIRE Adjacency matrix $A$, Laplacian $L = \diagmatrix \Lcombi \diagmatrix $, a coarsening ratio $r$ , preserved space $\preservedspace$, percentage number of nodes merged at one coarsening step : $n_e$
    \STATE $n_{obj} \gets \operatorname{int}(N - N\times r)$ the number of nodes wanted at the end of the algorithm.
    \STATE compute cost matrix $B_0 \gets VV^\top  \Lnorm^{-1/2}$ with $V$ an orthonormal basis of $\preservedspace$
    \STATE $Q \gets I_N$
    \WHILE{ $n \geq n_{obj} $ }
        \STATE Make one coarsening STEP $ l$
        \STATE Create candidate contraction sets.
        \STATE For each contraction $\mathcal{C}$, compute $\operatorname{cost}(\mathcal{C}, B_{l-1}, \Lnorm_{l-1}) = \frac{\lVert \Pi_C B_{l-1}(B^\top _{l-1}\Lnorm{l-1}B_{l-1})^{-1/2}\rVert_{\Lcombi_{\mathcal{C}}}}{|\mathcal{C}| -1}$
        \STATE Sort the list of contraction set by the lowest score
        \STATE Select the lowest scores non overlapping contraction set while the number of nodes merged is inferior to min($n-n_{obj}, n_e $)
        \STATE  Compute the binary $0$-$1$ matrix $\Qcombi_l$ intermediary lifting matrix with contraction sets selected
        \STATE $Q_{l-1} \Qnorm_{l} \gets \diagmatrix_{l-1}^{-1} \Qcombi_l \diagmatrix_l $
        \STATE $\Ploukas \gets \Qnorm_l^+ \Ploukas $
        \STATE $P_{MP} \gets \operatorname{reuniform}(\Qnorm_l^\top ) P_{MP}$
        %\STATE IF $P_{MP}$ cost THEN $B_l \gets  \operatorname{reuniform}(\Qnorm_l^\top )B_{l-1}$ ELSE $B_l \gets \Qnorm_l^+ B_{l-1}$ \AJ{TODO en effet ici $\Qcombi_l^+$ est un meilleur choix}
        %\STATE $P_{MP} \gets \operatorname{reuniform}(\Qnorm_l^\top ) P_{MP}$
        \STATE $B_l \gets \Qcombi^+ B_{l-1}$
        \STATE $A_l \gets \Qcombi^\top A_{l-1} \Qcombi -  \text{diag}\left((\Qcombi^\top A_{l-1}  \Qcombi)1_n\right) $
        \STATE $\Lnorm_l \gets \Qnorm^\top \Lnorm_{l-1} \Qnorm $
        \STATE $n \gets \operatorname{min}(n-n_{obj}, n_e )$
    \ENDWHILE
    \STATE $P_{MP} \gets (Q^\top Q)^{-1}Q^\top$
    \RETURN $A_c,\, \Qnorm,\, \Ploukas,\, P_{MP}$
\end{algorithmic}
\end{algorithm}

The terms $\Pi_\mathcal{C}$ and $L_\mathcal{C}$ denote specific projection of the contraction set. Their precise definitions are provided in Loukas work \cite{loukas2019graph}. We kept them unchanged in our experiments and defer any potential adjustments to future work.

In our adaptation we also add a parameter $n_e$ to limit the number of nodes contracted at each coarsening step. In one coarsening step, when a contraction set $\mathcal{C}$ is selected, we merge $|\mathcal{C}|$ nodes. In practice Loukas proposed in its implementation to force $n_e= \infty$ and coarsen the graph in one single iteration. We observed empirically that decreasing $n_e$ leads to improved results.

\paragraph{Candidate contraction Set.} Candidate contractions sets can take two main forms:  either pairs of nodes connected by an edge (referred to as the variation edges version), or the full neighborhood of each nodes (the variation neighborhood version). In practice, since neighborhoods tend to be large in our graph, this second option proves impractical for small coarsening ratios and typically leads to suboptimal results. We therefore rely on edge-based candidate sets, and adjust the parameter $n_e$ to control the greedy behaviour of the algorithm.

\section{Experiments hyperparameters}
For all the experiments, we preserve $K$ eigenvectors of the normalized self looped Laplacian $\Lnorm = \diagmatrix\Lcombi\diagmatrix$  with $\diagmatrix = (\operatorname{diag}(A1_N) +1)^{-1/2}$ with $K = 100$. We apply our adapted version of Loukas variation edges coarsening algorithm with $n_e= 10\%N$.

For the optimization hyperparameters as discussed previously in the appendix, we choose :
\begin{itemize}
    \item $P^*_{Q^\top}$ is initialized with the non zero coefficients of $P_{MP}$. We use a SGD optimizer with a momentum of 0.9, a learning rate $lr=0.05$ and $200$ epochs.
    \item $P^*_{g}$ is initialized with the non zero coefficients of $P_{MP}$. We use a SGD optimizer with a momentum of 0.9, a learning rate $lr=0.01$ and $200$ epochs.
    \item $P^*_{g,l_1}$ is initialized with the non zero coefficients of $\Prao$ and we use a $l_1$ penalty coefficient $\lambda = 0.01$ and a threshold of $0.001$ to enforce the sparsity. We use a SGD optimizer with a momentum of 0.9, a learning rate $lr=0.01$ and $200$ epochs.
\end{itemize}

\subsection{Hyperparameters for Tab.~\ref{tab:trainingresultsGCN}}
\label{app:gcnhyperparameters}

For the GCN, for both Cora and CiteSeeer we have 3 convolutional layers with the hidden dimensions $[256,128]$. We use an Adam Optimizer with a learning rate $lr = 0.01 $ and a weight decay $wd = 0.001$.

For the SGC model on Cora and Citeseer we make 2 propagations as preprocessing for the features.  We use an Adam Optimizer with a learning rate $lr = 0.1 $ and a weight decay $wd = 0.001$.

\subsection{Sparsity of the reduction matrices for Tab.~\ref{tab:trainingresultsGCN} }
\label{app:nnz}

\begin{table}[H]
\centering
\caption{Number of non zero elements for the different reduction matrices of Tab.~\ref{tab:trainingresultsGCN}  }
\vspace{10pt}
\small
    \begin{tabular}{ccccccc}
        \toprule
        \multirow{2}{*}{\#Non zero coefficients} & \multicolumn{3}{c}{Cora} & \multicolumn{3}{c}{ Citeseer } \\
        \cmidrule(lr){2-4} \cmidrule(lr){5-7} 
        $r$ & %$0.1$ & 
        $0.3$ & $0.5$ & $0.7$ & %$0.1$ &
        $0.3$ & $0.5$ & $0.7$  \\
        \midrule
        $\Ploukas$ & 
        2,485  & 2,485 & 2,485 & %now citeseer
        2,120 & 2,120 & 2,120  \\
        $P_{MP}$ & 
        2,485  & 2,485 & 2,485 & %now citeseer
        2,120 & 2,120 & 2,120  \\
        
        $\Prao$  &  
        4,315,200 & 3,080,144 & 1,846,584 &  %now citeseer
        3,136,320 & 2,235,527 & 1,341,714  \\
        
        $P_{Q^\top}^{*}$ &  
        2,485 & 2,485 & 2,485 & %now citeseer
        2,120 & 2,120 & 2,120  \\
         $P_{g}^{*}$&  
         2,261,446 & 2,202,595 & 1,596,687 & %now citeseer
         1,505,412 & 1,533,820 & 1,131,521  \\
         
         $P_{g,l_1}^{*}$&  
         2,822 & 3,003 & 2,881 & %now citeseer
         2,629 & 3,505 & 3,278 
         \\
         $\Prao$ (coeff > 0.001) &  
        583,431 & 766,439 & 580,971 &  %now citeseer
        537,353 & 594,807 & 421,027  \\
        
        $P_{g}^{*}$ (coeff > 0.001) &  
         5,735 & 46,080 & 99,288 & %now citeseer
         18,601 & 24,874 & 79,609  \\
\bottomrule
\end{tabular}
\label{tab:nnz_gnn}
\end{table}

%%%%%%%%%%%%%%%%%%%%%%%%%%%%%%%%%%%%%%%%%%%%%%%%%%%%%%%%%%%%

\newpage
\section*{NeurIPS Paper Checklist}

\begin{enumerate}

\item {\bf Claims}
    \item[] Question: Do the main claims made in the abstract and introduction accurately reflect the paper's contributions and scope?
    \item[] Answer: \answerYes{}% Replace by \answerYes{}, \answerNo{}, or \answerNA{}.
    \item[] Justification: This paper is mainly theoretical. The abstract and introduction present the theorems and their consequences.

\item {\bf Limitations}
    \item[] Question: Does the paper discuss the limitations of the work performed by the authors?
    \item[] Answer: \answerYes{} % Replace by \answerYes{}, \answerNo{}, or \answerNA{}.
    \item[] Justification: The limitations of the experiments, in particular the scalability of the method, are discussed in the main body of the paper.

\item {\bf Theory assumptions and proofs}
    \item[] Question: For each theoretical result, does the paper provide the full set of assumptions and a complete (and correct) proof?
    \item[] Answer: \answerYes{} % Replace by \answerYes{}, \answerNo{}, or \answerNA{}.
    \item[] Justification: The proofs are provided in the Appendix.

    \item {\bf Experimental result reproducibility}
    \item[] Question: Does the paper fully disclose all the information needed to reproduce the main experimental results of the paper to the extent that it affects the main claims and/or conclusions of the paper (regardless of whether the code and data are provided or not)?
    \item[] Answer: \answerYes{} % Replace by \answerYes{}, \answerNo{}, or \answerNA{}.
    \item[] Justification: The code is included as supplementary material, it can be executed on any computer. 

\item {\bf Open access to data and code}
    \item[] Question: Does the paper provide open access to the data and code, with sufficient instructions to faithfully reproduce the main experimental results, as described in supplemental material?
    \item[] Answer: \answerYes{}% Replace by \answerYes{}, \answerNo{}, or \answerNA{}.
    \item[] Justification: The datasets are available through Torch-Geometric, and the code is included as supplementary material. It only uses open-source Python libraries.

\item {\bf Experimental setting/details}
    \item[] Question: Does the paper specify all the training and test details (e.g., data splits, hyperparameters, how they were chosen, type of optimizer, etc.) necessary to understand the results?
    \item[] Answer: \answerYes{} % Replace by \answerYes{}, \answerNo{}, or \answerNA{}.
    \item[] Justification: All these details are mentioned in the setup paragraph of the experiments (Sec.~\ref{sec:experiments}) and in the Appendix. 

\item {\bf Experiment statistical significance}
    \item[] Question: Does the paper report error bars suitably and correctly defined or other appropriate information about the statistical significance of the experiments?
    \item[] Answer: \answerYes{} % Replace by \answerYes{}, \answerNo{}, or \answerNA{}.
    \item[] Justification: The results in Tab.~\ref{tab:trainingresultsGCN} are averaged over ten runs and reported with standard deviations.

\item {\bf Experiments compute resources}
    \item[] Question: For each experiment, does the paper provide sufficient information on the computer resources (type of compute workers, memory, time of execution) needed to reproduce the experiments?
    \item[] Answer: \answerYes{} % Replace by \answerYes{}, \answerNo{}, or \answerNA{}.
    \item[] Justification: This is small-scale code, reproducible without specific computational resources.
    
\item {\bf Code of ethics}
    \item[] Question: Does the research conducted in the paper conform, in every respect, with the NeurIPS Code of Ethics \url{https://neurips.cc/public/EthicsGuidelines}?
    \item[] Answer: \answerYes{} % Replace by \answerYes{}, \answerNo{}, or \answerNA{}.
    \item[] Justification: We confirm that our paper complies with the NeurIPS Code of Ethics. This is mainly a theoretical paper, and the code uses only open-source Python libraries.

\item {\bf Broader impacts}
    \item[] Question: Does the paper discuss both potential positive societal impacts and negative societal impacts of the work performed?
    \item[] Answer: \answerNA{} % Replace by \answerYes{}, \answerNo{}, or \answerNA{}.
    \item[] Justification: Given the theoretical nature of our work, we do not expect any direct societal impact. Future work on the scalability of graph coarsening algorithm may raise these questions, but this is beyond the scope of this paper.
    
\item {\bf Safeguards}
    \item[] Question: Does the paper describe safeguards that have been put in place for responsible release of data or models that have a high risk for misuse (e.g., pretrained language models, image generators, or scraped datasets)?
    \item[] Answer: \answerNA{} % Replace by \answerYes{}, \answerNo{}, or \answerNA{}.
    \item[] Justification: This paper does not present such models.

\item {\bf Licenses for existing assets}
    \item[] Question: Are the creators or original owners of assets (e.g., code, data, models), used in the paper, properly credited and are the license and terms of use explicitly mentioned and properly respected?
    \item[] Answer: \answerYes{} % Replace by \answerYes{}, \answerNo{}, or \answerNA{}.
    \item[] Justification: The original papers for Cora and Citeseer are cited. Further details on these datasets are provided in the Appendix.

\item {\bf New assets}
    \item[] Question: Are new assets introduced in the paper well documented and is the documentation provided alongside the assets?
    \item[] Answer: \answerNA{} % Replace by \answerYes{}, \answerNo{}, or \answerNA{}.
    \item[] Justification: This paper does not release new assets.

\item {\bf Crowdsourcing and research with human subjects}
    \item[] Question: For crowdsourcing experiments and research with human subjects, does the paper include the full text of instructions given to participants and screenshots, if applicable, as well as details about compensation (if any)? 
    \item[] Answer: \answerNA{} % Replace by \answerYes{}, \answerNo{}, or \answerNA{}.
    \item[] Justification: This paper does not involve crowdsourcing nor research with human subjects.

\item {\bf Institutional review board (IRB) approvals or equivalent for research with human subjects}
    \item[] Question: Does the paper describe potential risks incurred by study participants, whether such risks were disclosed to the subjects, and whether Institutional Review Board (IRB) approvals (or an equivalent approval/review based on the requirements of your country or institution) were obtained?
    \item[] Answer: \answerNA{} % Replace by \answerYes{}, \answerNo{}, or \answerNA{}.
    \item[] Justification: This paper does not involve crowdsourcing nor research with human subjects.

\item {\bf Declaration of LLM usage}
    \item[] Question: Does the paper describe the usage of LLMs if it is an important, original, or non-standard component of the core methods in this research? Note that if the LLM is used only for writing, editing, or formatting purposes and does not impact the core methodology, scientific rigorousness, or originality of the research, declaration is not required.
    %this research? 
    \item[] Answer: \answerNA{} % Replace by \answerYes{}, \answerNo{}, or \answerNA{}.
    \item[] Justification: The core method development in this research does not involve LLMs as any important, original, or non-standard components.

\end{enumerate}

\end{document}